\documentclass{svjour3}
\usepackage{lmodern}
\journalname{Machine Learning}
\usepackage{natbib}
\usepackage[utf8]{inputenc}
\usepackage{amsmath,amsfonts,amssymb,bm}
\usepackage{graphicx,booktabs,tabularx,array,multirow,multicol}
\usepackage{algorithm,algorithmic}
\usepackage{enumerate,comment}
\usepackage{url,xspace,listings,xcolor}
\usepackage[T1]{fontenc}

\usepackage{appendix}
\usepackage{caption, subcaption}
\usepackage{hyperref}
\usepackage{tikz}
\usetikzlibrary{calc}

\DeclareRobustCommand{\blue}[1]{#1} 
\newcommand{\strike}[1]{}

\floatname{algorithm}{Procedure}
\newtheorem{mydefinition}{Definition}
\newtheorem{myexample}{Example}
\newtheorem{myremark}{Remark}
\newtheorem{mylemma}{Lemma}

\newtheorem{myproposition}[mylemma]{Proposition}
\newtheorem{myprop}{Proposition}

\newcommand{\gray}[1]{\textcolor{gray}{#1}}
\newcommand{\SH}{\texttt{SearchHyp}\xspace}

\newcommand{\GM}{\texttt{GenMol}\xspace}
\newcommand{\Gen}{$\lambda$\texttt{Gen}\xspace}

\hypersetup{
    colorlinks=true,
    linkcolor=blue,
    citecolor=blue,
    filecolor=magenta,
    urlcolor=blue
}

\begin{document}
\title{Symbolic Neural Generation with Applications to Lead Discovery in Drug Design}

\author{Ashwin Srinivasan \and 
        Tirtharaj Dash \and 
        A Baskar \and 
        Michael Bain \and 
        Sanjay Kumar Dey \and 
        Mainak Banerjee
}

\authorrunning{Srinivasan et al.} 

\institute{Ashwin Srinivasan \at 
            Dept. of Computer Science \& Information Systems and APPCAIR \\
            BITS Pilani, K K Birla Goa Campus, India\\
            \email{ashwin@goa.bits-pilani.ac.in} \\
            \emph{Corresponding Author}
            \and
            Tirtharaj Dash \at 
            Dept. of Computer Science \& Information Systems \\
            BITS Pilani, K K Birla Goa Campus, India \\
            \email{tirtharaj@goa.bits-pilani.ac.in} \\[1ex]
            Department of Biochemistry, \\
            University of Cambridge, Cambridge, UK \\[1ex]
            \emph{Note:} Part of this work was carried out while the author was at the University of Cambridge, UK. The major revision of this work was completed at BITS Pilani.
            \and
            A. Baskar \at 
            Dept. of Computer Science \& Information Systems\\
            BITS Pilani, K K Birla Goa Campus, India\\
            \email{abaskar@goa.bits-pilani.ac.in} \and
            Michael Bain \at
            School of Computer Science and Engineering \\
            University of New South Wales, Sydney\\
            \email{m.bain@unsw.edu.au}
            \and
            Sanjay Kumar Dey \at 
            Dr. B.R. Ambedkar Center for Biomedical Research\\
            University of Delhi, New Delhi, India\\
            \email{skdey@acbr.du.ac.in}
            \and
            Mainak Banerjee \at 
            Department of Chemistry\\
            BITS Pilani, K.K. Birla Goa Campus, India\\
            \email{mainak@goa.bits-pilani.ac.in}
}

\maketitle

\begin{abstract}
We investigate a relatively under-explored class of hybrid neurosymbolic models that integrate
symbolic learning with neural reasoning to construct data generators meeting formal correctness 
criteria. In \textit{Symbolic Neural Generators} (SNGs), symbolic learners
examine logical specifications of feasible data from a small set of
instances—sometimes just one. Each specification in turn constrains the conditional
information supplied to a neural-based generator, which rejects any instance
violating the symbolic specification.
Like other neurosymbolic approaches, SNG exploits the complementary strengths of symbolic and 
neural methods. The outcome of an SNG is a pair $(H, X)$, where $H$ is a
 symbolic description
of feasible instances constructed from data, and $X$ a set of generated new instances
that satisfy the description.  We introduce a semantics 
for such systems, based on the construction of appropriate \textit{base} and \textit{fibre}
partially-ordered sets combined into an overall partial order.
We implement an SNG combining a restricted form of Inductive Logic Programming (ILP) with a 
large language model (LLM) and evaluate it on early-stage drug design. Our main interest is
the description and the set of potential inhibitor molecules generated by the SNG.
On benchmark problems -- where drug targets are well understood -- SNG performance is
statistically comparable to state-of-the-art methods. On exploratory problems with poorly 
understood targets, generated molecules exhibit binding affinities on par with leading clinical 
candidates. Experts further find the symbolic specifications useful as preliminary filters, 
with several generated molecules identified as viable for synthesis and wet-lab testing.
\end{abstract}


\section{Introduction}
\label{sec:intro}

Consider the following real problem\footnote{Biochemical terms used
throughout this paper are briefly explained in
Appendix~\ref{app:biochem}.}:

\begin{quote}
We are interested in generating a set of small molecules that
can bind to a known protein. We know something about the protein: its amino-acid sequence, its role
in causing some disease, an approximate section of the sequence where the small molecule should bind,
and so on. We also know 5 small molecules (inhibitors) that are known to bind to this protein, of which
1 is known to have toxic side-effects.
We also know, from our previous chemical knowledge, that to be easily synthesised
we would like molecules that contain a particular scaffold,
whose molecular weights are not too high or too low, predicted toxicity values are as low
as possible, and predicted binding affinity to the protein is as high as possible. Additional
constraints on the molecule are not known, since the 3 dimensional structure of the target site
has only recently become available.
Can we generate 10 additional possible inhibitors?
\end{quote}

\noindent
Conceptually, this can be seen as a special case of the general problem of identifying
elements of a set:

\[
    {\cal X} = \{x: x \in {\cal U}, \Phi(x) \mbox{ is true}\}
\]

\noindent
where ${\cal U}$ is a set consisting of all possible instances of relevance -- small molecules,
for instance -- and $\Phi(x)$ is a predicate that is true of the specific instances of interest.
Given the remarkable abilities of large pre-trained models,
if $\Phi(\cdot)$ is known, it is conceivable that we may indeed be able to generate
molecules directly. Neural-generators are stochastic, and some of the generated molecules
may not actually satisfy $\Phi(\cdot)$, but we can resort to some kind of rejection-sampling
until we meet our requirement. The obvious difficulty is, of course, that for many
problems, ${\cal U}$ can be very large (for example, there may be up to ${10}^{60}$ small molecules),
and rejection-sampling can become quite inefficient.
The real issue though is that for most complex real problems, we do not know $\Phi$.
Some  options we could try are:
    \begin{itemize}
        \item We could consider fine-tuning an existing generative model with 
            the data instances we already know. This runs into the difficulty that the data are often
            observations of phenomena that are  rare, and therefore we usually have very few
            instances -- in the order of 10s, rather than the several 100s or 1000s needed for
            effective fine-tuning.
            In the event, we will probably be left with a generator that is largely unaffected
            by the few data instances we have;
        \item We could consider prompt-engineering for an LLM. That is, assuming
            there exists some description $p$ of the set ${\cal X}$ for
            some LLM $\lambda$, we attempt to find $p$ by trial-and-error. There
            are two issues here: it is not clear that the text-based description
            $p$ is any easier to identify than the formal description $\Phi$; and
            it is not clear that it would be any more precise;
        \item We could consider exploiting the few-shot learning ability of an LLM
            \citep{Brow:etal:x:2020} by providing the data instances we know as part of the
            context information for the LLM.
            But, how do we then know whether the generated instances are indeed from
            the set of interest? In-context learning usually works best within an iterative loop that
            updates the context with the result of some validation mechanism providing
            feedback. What would this be here? For specialised problems, human feedback
            may be both difficult to obtain and unlikely to be
            as helpful as with routine conversations.
    \end{itemize}

\noindent
The issue is this: with modern pre-trained large neural models 
capable of approximating a vast range of probability distributions, the difficulty is
in finding verifiable ways of constraining them to `focus'.
However, this still begs the question of whether, \textit{in principle\/}, 
the task can be achieved using a purely symbolic or purely connectionist approach 
(``unified'' approaches in \citep{hilario2013overview, Hila:etal:p:1994}).

In principle, the answer is ``yes''. For a symbolic generator, we would need
to identify a symbolic description
$S$ to include probability values over the elements of ${\cal U}$:
logic programs with a distributional semantics like \citep{sato1997prism} or
\citep{de2007problog} are examples of this.
The symbolic description can then be directly  used for generating  instances in
by sampling. The difficulty, however, is that symbolic representations are usually
not fine-grained enough to capture very complex probability distributions; which are better
approximated with the high-dimensional real-vector encodings (embeddings) used by
neural networks.
This is especially apparent with using language models for accessing complex conditional probability
distributions. 

What about a purely connectionist approach? Again, in principle, a neural network can
encode a predicate such as $\Phi(\cdot)$. The difficulty here is that this encoding
becomes increasingly more approximate
if obtained from very few data samples (large language models have proved remarkably
successful at identifying appropriate text-responses with very few examples; it is not
clear as yet if this extends robustly to logical formulae). The lack of a precise, verifiable
formal description from a neural model also presents a challenge. Of
course, if human-understandability of the description is also required, then we face
well-known difficulties with a neural encoding.
Thus, the continued interest in hybrid systems is not
for theoretical but for practical reasons. Hybrid systems can be engineered from
existing modules; are easier to maintain,
and are more amenable to controlled studies, especially if each component is largely
concerned with a specific function. 

In this paper, we propose using a symbolic component that
examines possible approximations for $\Phi(\cdot)$ by constructing
definitions for a predicate $\Sigma(\cdot)$. Each such definition is accompanied by a
set of instances obtained from a neural generative model
for which $\Sigma(\cdot)$ is true. Figure \ref{fig:venn} shows the outcome we would
like ideally, and what we would have to settle for in practice.
This requires us to deal with
a \textit{dual-alignment} problem: we would like the symbolic
component to find an $S$ that aligns well with ${\cal X}$; and a neural
generator that identifies an $N$ that aligns well with $S$
(denoted by the set $X$).
 This combination using a symbolic learner and a neural-based generator constitutes a
 \textit{Symbolic Neural Generator\/}, or SNG.\footnote{We will sometimes use
 SNG to stand for the computation process of \textit{symbolic neural generation}, with
 the context making it clear whether we mean a system or the process.} Figure \ref{fig:venn}
 shows what we would like ideally, and what often results in practice.
 In the taxonomy proposed in \citep{hilario2013overview}, SNGs are a
 \textit{hybrid} neurosymbolic\footnote{There have been several variations
of this term, such as ``neural-symbolic''~\citep{Garc:etal:b:2009}. Current usage seems to have
converged on the single word ``neurosymbolic''~\citep{Garc:Lamb:j:2023}.} approach
 since they contain distinct neural and symbolic components.

\begin{figure}[!htb]
\begin{subfigure}{0.33\textwidth}
    \centering
    \includegraphics[height=0.225\textheight]{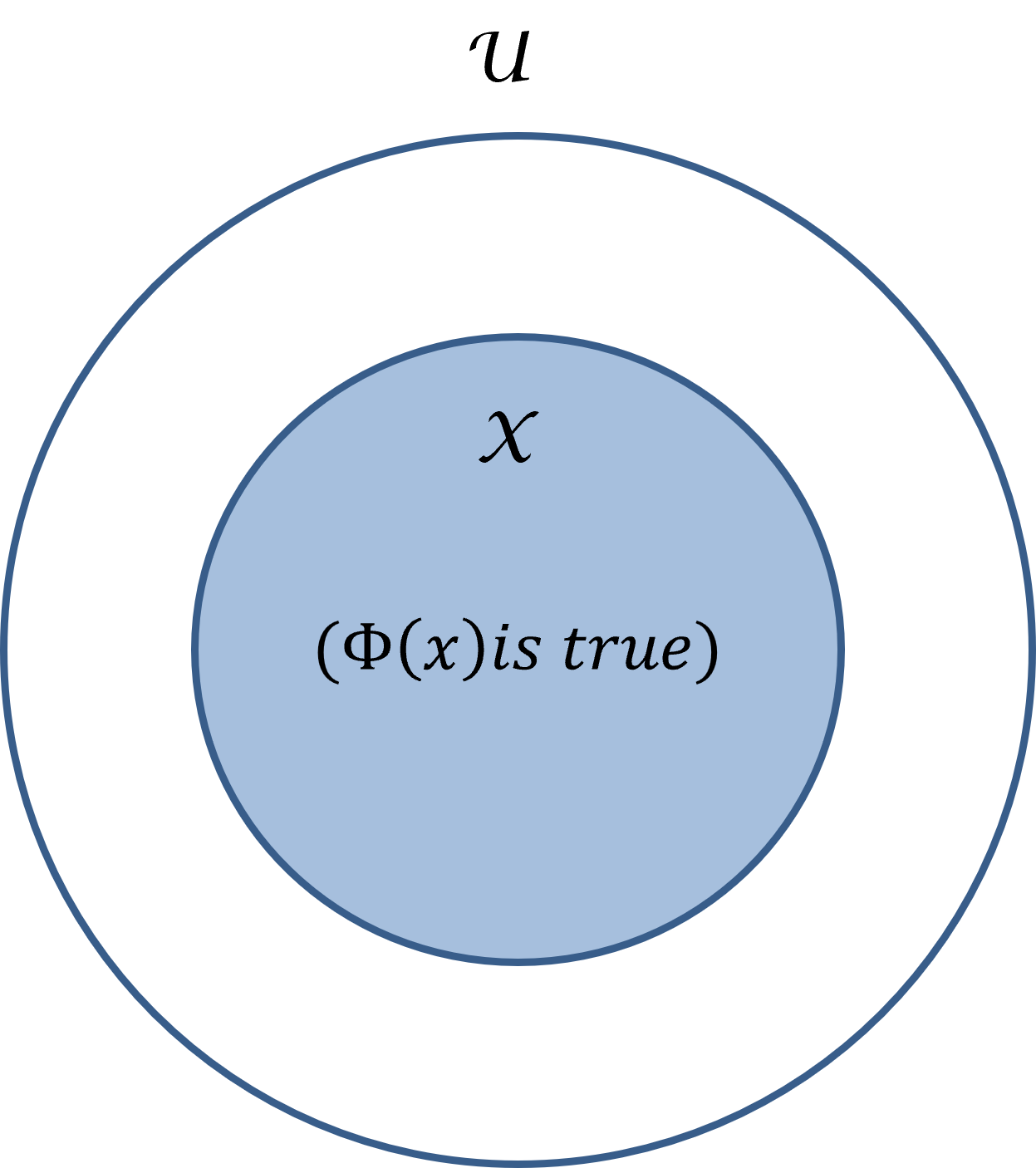}
    \caption{Ideal}
\end{subfigure}%
\hfill
\begin{subfigure}{0.33\textwidth}
    \centering
    \includegraphics[height=0.225\textheight]{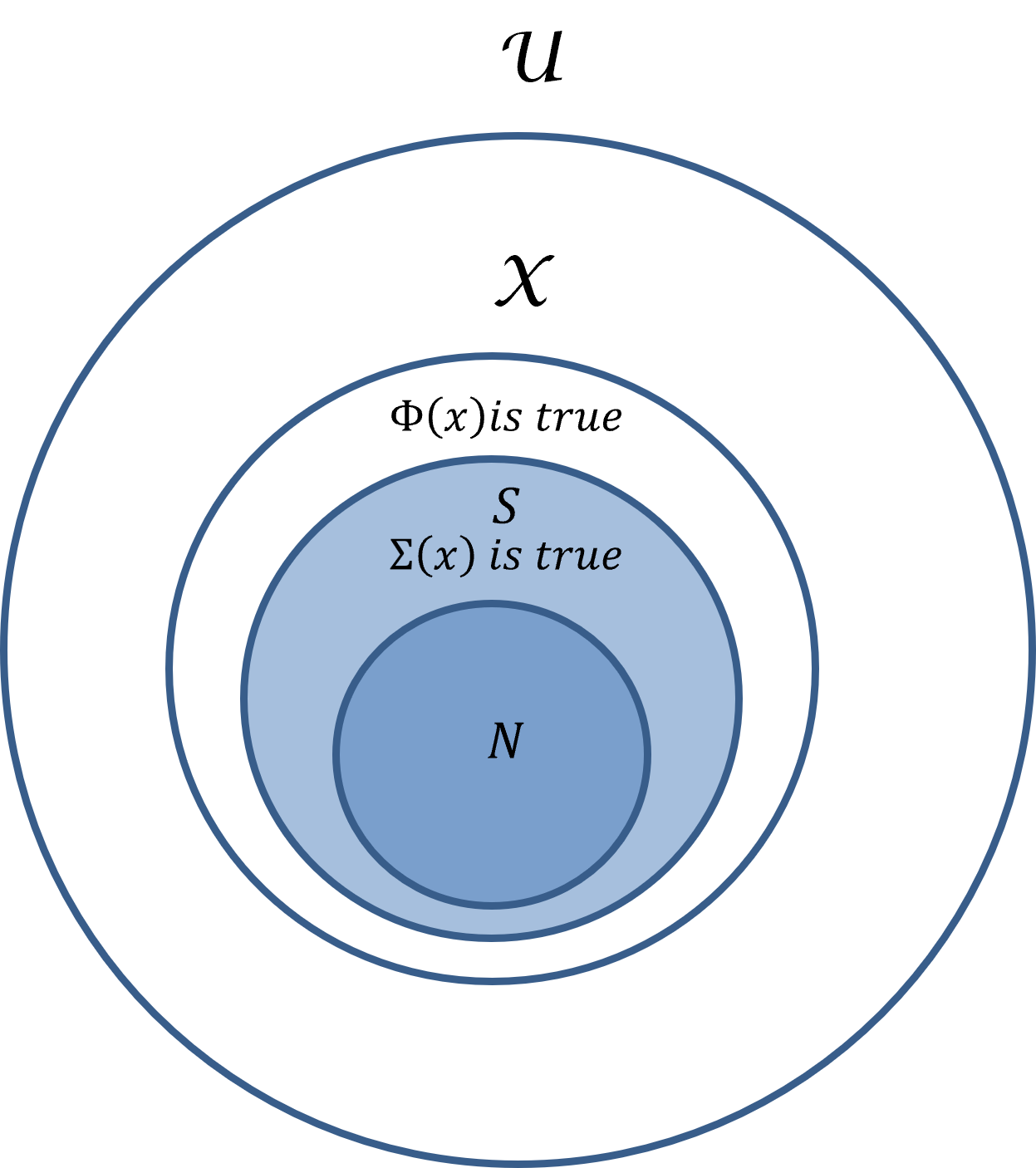}
    \caption{Ideal SNG}   
\end{subfigure}%
\hfill
 \begin{subfigure}{0.33\textwidth}
 \centering
    \includegraphics[height=0.225\textheight]{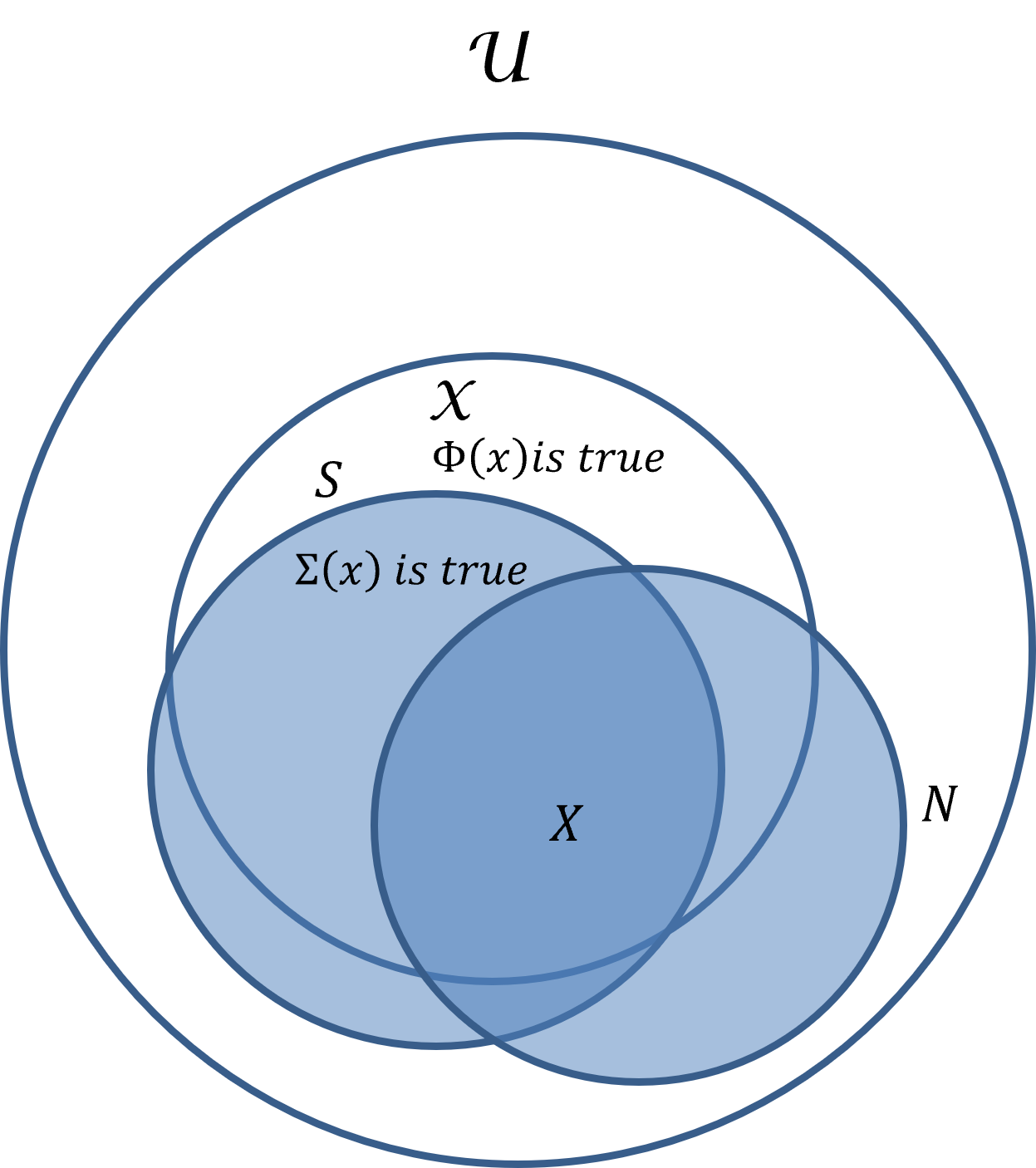}
    \caption{Practical SNG}   
\end{subfigure}
\caption{(a) Ideally, we would like to generate instances from the set of instances
    for which $\Phi(x)$ is true; (b) When $\Phi(\cdot)$ is not known, we approximate $\Phi(\cdot)$
    by $\Sigma(\cdot)$, obtained using the hypothesis from
        a symbolic learner. We want to sample instances
    efficiently from $S$. $N$ is the set of instances obtained from a neural-based generator.
    For an ideal SNG, $N \subseteq S \subseteq {\cal X}$;
    (c) In practice, the symbolic learner may not be perfect, and
        the neural-generator only has an approximate model of the conditional distribution. 
        The set $X$ is the set of instances generated that are in $S$.}
\label{fig:venn} 
\end{figure}

\noindent
 We see the paper contributing in the following ways to research into
 neurosymbolic systems:
 
 \begin{itemize}
    \item We identify a class of hybrid neurosymbolic systems -- called
    Symbolic Neural Generators, or SNGs -- which draws on
        the strengths of symbolic learning and neural-based generative models.
        Potential symbolic descriptions are examined, and at the same time:
        (a) act as conditioning information for neural-based encodings
        of probability distributions; and (b) verify samples of new data instances
        generated by the neural component. The output of an
        SNG is a human-readable
        symbolic description and a set of instances consistent with that description.
     \item We provide a novel semantic framework for SNGs
        based on poset semantics. In this formulation, each
        element of a partial ordering over symbolic hypotheses
        is associated with a (second) partial order over
        neural-generated instances that are consistent with
        the hypothesis. The semantics provides a precise specification of the codomain of an SNG. 
    \item We implement and test SNGs on the real-world problem of 
        generating potential
        inhibitors for protein-targets. This constitutes the problem of `lead-discovery',
        and realistic settings have the following characteristics. Given:
        (a) a few known inhibitors (in our main case study, only 5), and
        (b) some amount of prior biological and chemical knowledge; we want:
        (c) human-readable constraints on potential new inhibitors, and
        (d) proposals for potentially new inhibitors for the target. Our
        experiments show that SNG performs creditably on benchmarks from the literature,
        and -- on an open problem -- provides results that are understandable and
        interesting to a synthetic chemist and a structural biologist.
 \end{itemize}

\section{Background: Hybrid Neurosymbolic Modelling}
\label{sec:rel_work}

 
The idea of neurosymbolic modelling for AI is not new; as has often been noted, 
the seminal paper of~\cite{McCu:Pitt:j:1943} 
rests on the observation that ``the activity
of any neuron can be represented as a
proposition''. Thus arises an entire class of systems comprising a purely
connectionist approach that approximates the patterns of
categoric inference arising when using (often propositional, but not always)
logical representations, as well
as the patterns of plausible  inference that arise when using probabilistic
representations. There is now a very large literature on this kind
of system -- called a \textit{unified} neurosymbolic system in
\citep{hilario2013overview,Hila:etal:p:1994} -- and this forms a prominent
part of the current landscape of neurosymbolic AI. We will not attempt an
exhaustive review of unified neurosymbolic literature.
 For extensive treatments of that
see:~\citep{Garc:etal:b:2009,hitzler2022neural,Garc:Lamb:j:2023,DeSm:DeRa:x:2025,luc:deeplog}
We note that while much of the work focuses on engineered solutions,
we are now seeing the emergence of  more abstract, mathematical
perspectives on what represents a unified neurosymbolic system.
Thus, \cite{Oden:Garc:j:2025} are concerned with identifying the
conditions of semantic equivalence between architectures and representations;
and \cite{DeSm:DeRa:x:2025} provides the mathematical underpinning
of a uniform inference mechanism for reasoning
in systems containing both symbolic and neural components.
 Also under the category of unified systems,
 we do not present here any of the vast array of modelling and inference
methods that continue to be developed using 
probability theory of continuous and discrete
random variables and their associated distributions. Ultimately
under the category of unified symbolic approaches, it is possible to provide a probabilistic
semantics to (conditional and unconditional) generative models, irrespective of
whether distributions are encoded by neural networks, probabilistic programs, graphical
models or the like (we refer the reader to \citep{murphy:pml2}, Chs. 20--28
for an extensive treatment) and to~\citep{Marr:etal:j:2024}
for the related area of statistical relational learning.

In this paper, we are concerned with
the alternative form of \textit{hybrid neurosymbolic systems}
that contains distinct neural and symbolic
components. It is helpful to categorise hybrid systems  based on the
 the principal function of each component in the system (Fig.~\ref{fig:categ}).\footnote{
 In practice, neither component does just one function or the other.
 For example, any non-trivial learning (induction) usually
requires reasoning (deduction), and any sufficiently complex reasoning would usually
be accompanied by some trial-and-error learning. Therefore, this categorisation is
necessarily a simplification.}
SNG as we propose it refers to a sub-class of systems in Category (B). 

\begin{figure}[htb]
\begin{minipage}{0.3\textwidth}
\begin{subcaption}
\centering
\begin{tabular}{cc|c|c|} 
      &\multicolumn{1}{c}{\mbox{}} & \multicolumn{2}{c}{Symbolic} \\
      & \multicolumn{1}{c}{\mbox{}} & \multicolumn{1}{c}{$Reason$}    & \multicolumn{1}{c}{$Learn$}  \\ \cline{3-4}
    \multirow{4}{*}{\rotatebox{90}{Neural}} & $Reason$ & $(A)$ & $(B)$ \\
             &              &           &        \\ \cline{3-4}
             &  $Learn$     & $(C)$     & $(D)$  \\ 
             &              &           &        \\ \cline{3-4}
\end{tabular}
\label{fig:categ_table}
\end{subcaption}
\end{minipage}%
\hfill
\begin{minipage}{0.55\textwidth}
\begin{subcaption}
\centering
{\scriptsize{
\begin{tabular}{cl}\hline
Category & Examples \\ \hline
(A) & Neural theorem proving; Plan verification;\\
    & Safe neural controllers;  \\
    & Neural proof search \\ \hline
(B) & Grammar learning; \\
    & Learning explanations for neural behaviour\\
    & ILP with neural inference; \\
    & Symbolic models for neural generators  \\ \hline
(C) & Schema learning with ontologies; \\
    & Deep RL with symbolic constraints \\
    & Semantic-loss based classification;\\
    & Design generation with spatial constraints \\ \hline
(D) & Hybrid KG construction; \\
    & Differentiable ILP; \\
    & Symbolic concepts from neural representations\\
    & Program synthesis with learned primitives \\ \hline
\end{tabular}
}}
\label{fig:categ_examples}
\end{subcaption}
\end{minipage}
\caption{(a) A categorisation of hybrid neurosymbolic systems that consist of distinct
    Neural and Symbolic components. The primary role of each component is to
    $Learn$ or to $Reason$; (b) Examples of neurosymbolic systems in each category.}
\label{fig:categ}
\end{figure}

We are concerned with a restricted class of hybrid
neurosymbolic systems whose task it is to `generate' new data instances using (models for)
conditional or joint probability distributions.
In the rest of this section, we will
highlight some recent relevant work that satisfies these criteria,
in the categories shown in Fig.~\ref{fig:categ_table}.

An emerging trend for generative models in Category A (Neural-Reasoning, Symbolic-Reasoning)
is the use of pre-trained LLMs that
invoke ``tools''~\citep{Kautz:2024:Tools} to verify correctness of instances
generated by the LLM. In particular, such hybrid systems can
invoke reasoning systems -- theorem-provers, for example -- implementing sound inference in symbolic logic~\citep{cheng2025empowering_alt}.
For example, in a system like LINC~\citep{Olau:etal:x:2024} an LLM is used to translate problem statements in natural language to expressions in first-order logic, and then attempts to
find proofs using a theorem prover.
LLMs with tools are able to exceed the formal reasoning capabilities of LLMs, which can struggle on logical reasoning tasks without additional training~\citep{Li:etal:p:2024:LLM_Generation_for_Math,Qi:etal:2025:LLMs_Logic_FineTuning}.
The AlphaGeometry2 system~\citep{Chervonyi:etal:2025:AlphaGeometry} uses a language model to take in natural language statements of mathematical problems and generate formal expressions of problem facts in a domain-specific language for IMO geometry problems. A deductive database algorithm computes the closure of these facts, which is searched in parallel for solutions, with a language model used to generate the proofs. Closely related to the work here,
the mechanism of language models with logical-feedback (LMLF) in \citep{brahmavar2024generating}
can be seen as an LLM equipped with a symbolic reasoner. The symbolic component performs two tasks: (a) using prior knowledge it progressively deduces new constraints that in turn
modify the context of the LLM; and (b) it verifies that instances generated by the LLM
satisfy the constraints. LLMs are used here for generation and to provide explanations
in a natural or specialised language. The symbolic reasoner can also provide additional
justification, by displaying the reasoning steps used to verify the output generated
by the LLM.
 
While there is a relative paucity of generative hybrid systems
in Category B (Neural-Reasoning, Symbolic-Learning), some special-purpose systems do exist. 
For example, building on earlier work in which pre-defined grammars were used~\citep{Kusner:etal:p:2017}, a symbolic generative grammar defined on molecular hypergraphs was trained utilising probabilistic learning in~\citep{Guo:etal:x:2022}.
The approach uses bottom-up grammar learning, where a molecule, represented as a hypergraph on the left-hand side of a production rule, generates a new molecule (by adding a
hyperedge) on the right-hand side.
A pre-trained neural network is used as a molecular feature selector and a probabilistic model is conditioned on the data using Monte Carlo sampling and gradient ascent to maximise score across molecular metrics.
Molecules are then probabilistically sampled from this grammar, providing a data-efficient (i.e., it learns from a few molecules) and interpretable molecular generator. 
More recently, a pre-trained multi-modal foundation model was used to guide the symbolic grammar learning, based on prompting and neural reasoning, which is then used to generate molecules~\citep{Sun:etal:x:2025:MMFM_MolGen}. Although a symbolic theory is learned, both
learning and reasoning are done by neural systems, and this recent
work should be treated as a unified neurosymbolic system. 
Within robot planning, actions can be modelled as sets of predicates on the environment that represent pre-conditions or post-conditions.
In~\citep{Liang:etal:x:2025:ExoPredicator} an LLM generates predicates corresponding to a robot trajectory in the environment, where a plan either succeeded or failed.
Generated predicates are converted to Python code, filtered against an API and used in a search for state abstractions.
This corresponds to a form of program synthesis,
in which a causal process world model for planning is learned using a probabilistic approach.
This approach can be characterised as neural reasoning (generation) and symbolic (probabilistic) learning.

In contrast, there is more done on general-purpose generative
hybrid approaches in Category C (Neural-Learning, Symbolic-Reasoning).
For example, diffusion models are widely used to generate images, but can also be applied to generate discrete outputs such as language, or molecules as SMILES strings.
However it is difficult to control generation to ensure certain requirements on the
outputs are satisfied.
A neurosymbolic approach to diffusion is in~\citep{christopher2025neurosymbolicgenerativediffusionmodels}.
To check for toxicity in molecule generation several black-box filters implemented
in RDKit were selected.
Their outputs are then used in constraints in a convex optimization solver which is integrated into the diffusion process.
SPRING~\citep{Jacobson:Xue:j:2025:NeSy_DesignGen} is a related approach which
incorporates a symbolic spatial reasoning module into diﬀusion-type models.
A user-defined design language on spatial predicates and object relations 
enables greater control of generated designs.

Also in Category C is the important class of neural networks trained to generate
molecules that are constrained by symbolic specifications, such as target properties 
of the molecule, or grammars or graph structures capturing chemical knowledge,
e.g.,~\citep{lim2018molecular,liu2018constrained,dash2021using}.
A neurosymbolic relational generative model in this category is the Neural Markov Logic Network~\citep{marra2021neural}, which generalises Markov Logic Networks by replacing hand-specified weighted first-order rules with a neural potential function defined over $k$-sized fragments of a relational structure. The model is trained by maximum-likelihood with Gibbs sampling for the partition function, and chemical validity in the molecule generation experiment was enforced through rejection sampling using RDKit.
This category also includes pre-trained LLMs fine-tuned with data expressed in molecular
languages such as SMILES, and constrained to ensure correct generation with symbolic models~\citep{zhang2025scientific}.
A recent approach~\citep{Zhou:etal:j:2025} uses molecular analysis tools, some based on
pre-trained statistical machine learning, to generate a range of relevant properties 
which are converted to text using natural language templates. These molecule-text pairs 
are added to the training set to fine-tune an LLM, which can then generate new 
molecules from text prompts with high probability of having the relevant properties.

Category D (Neural-Learning, Symbolic-Learning)
represents, in some sense, the most complex category of hybrid neurosymbolic
systems in Fig.~\ref{fig:categ_table}.
Joint learning of neural and symbolic models is also a feature
of \cite{mugg:abindraw}: although neither model is generative, there is no
reason in principle why meta-interpretive learning (the approach used in that
paper) cannot be used to learn generative neural models guided by
symbolic hypotheses. In principle, the unified framework of
\cite{DeSm:DeRa:x:2025} allows the joint learning of neural and symbolic models. It
is unclear whether the related implementation DeepLog \citep{luc:deeplog} -- not
to be confused with the identically named symbolic learner
described in \citep{muggleton2023deeplog} -- has the
operations required to achieve this, but we do not see any difficulty in
principle.

\section{Symbolic Neural Generation}
\label{sec:sng}

Consider again the problem of identifying instances from the set:
\[
    {\cal X} = \{x: x \in {\cal U}, \Phi(x)\ \mbox{ is true }\}
\]

\noindent
with the condition that $\Phi(\cdot)$ is not known. However, we do
have some instances of ${\cal U}$ for which membership in ${\cal X}$ (or otherwise)
\emph{is} known. There may also be some problem-specific background knowledge that may
be of help. In symbolic neural generation, we approach this problem by using what
we know to examine an approximation $\Sigma(\cdot)$ to $\Phi(\cdot)$
and using this to constrain the subsequent selection of instances using a neural-based
sampler. This two-component approach can be seen as an instance of a
hybrid system in Category B of the previous section. 
We will now attempt a clearer understanding of this kind of system.

\subsection{Poset Semantics for SNGs}
\label{sec:sem}

We present a simple semantics for hybrid neurosymbolic systems
that perform symbolic neural generation. 
The semantics serves two roles: (i)~it specifies the \emph{codomain}
of any SNG system -- the kind of object an SNG is required to return,
namely a pair $(H,X)$ having a symbolic hypothesis $H$ and a set $X$
of valid neural samples consistent 
with $H$; 
and (ii)~it forms the basis of
a proof of correctness for the implementation
(Prop.~\ref{prop:gencorrect} and
Remark~\ref{rem:sng}). We rely throughout on the
partial order $\geq_{\mathcal H}$ on hypotheses
and, for each hypothesis $H \in {\mathcal H}$, the set of
subsets of the instances satisfying $H$ in the universe of
instances.\footnote{
Later we suggest alternatives and extensions
that apply to wider classes of hybrid neurosymbolic systems.}

Let $\mathcal U$ denote a fixed finite universe of instances.
Let $B$ be background knowledge and let
$\mathcal H$ be a family of symbolic hypotheses.  Assume each hypothesis $H$ in $\mathcal H$ has a definition of a fixed unary predicate
$\Sigma$ which is defined over $\mathcal U$.

\begin{mydefinition}[Extension of $H$]
 \label{def:ext}
 Let $H \in {\cal H}$.
 The semantic extension of  $H$ given $B$, denoted
 by $\operatorname{ext}(H|B)$, is the set $\{x: x \in {\cal U}, \, B \wedge H \models \Sigma(x)\}$.
 We will call this simply the extension of $H$ and
 denote it by $\operatorname{ext}(H)$ if $B$ is clear from the context.
 \end{mydefinition}

\noindent
We now describe the structure of SNGs, starting with the
\textit{base poset} $(\mathcal H,\le_{\mathcal H})$.

\begin{mydefinition}[Ordering over ${\cal H}$]
 \label{def:pohyp}
 Let $B \in {\cal B}$ be background knowledge and
 $H_1, H_2 \in {\cal H}$ be hypotheses.
 We define $H_1 \geq_{\cal H} H_2$ as $\operatorname{ext}(H_1|B) \supseteq \operatorname{ext}(H_2|B)$.
 \end{mydefinition}

\begin{myproposition}
$({\cal H},\geq_{\cal H})$ is a partially ordered set.
\end{myproposition}
\begin{proof}
Follows trivially from the reflexive, anti-symmetric and transitive
properties of subset-inclusion.
\end{proof}


\noindent
We will call $(\mathcal H,\ge_{\mathcal H})$ the \textit{base} poset.
For each $H\in\mathcal H$ we associate a \emph{fibre-poset}
which depends on the neural component. For the present, let the neural generator be a function
$f_N: {\cal H} \times Z \to {\cal U}$ where $Z$ is some latent space, and let
\[
N(H) ~=~ \operatorname{im}(f_N(H)) ~=~ \{f_N(H,z): z \in Z\} \subseteq {\mathcal U}
\]
That is, $N(H)$ is a the set of all instances that can be generated in principle by the
neural component using the hypothesis  $H$.\footnote{
Thus, if the neural component is a stochastic model that
defines a probability distribution $P_N$ over $\mathcal U$,
$N(H) = \{{u \in \cal U}: P_{N(H)}(u) > 0 \}.$}
Later we consider extensions that account for more subtle variants.
\begin{mydefinition}[Fibre-poset of a base element]
\label{def:fiber}
For each $H \in {\mathcal H}$, we define a
corresponding fibre-poset:
\[
F(H) \;=\; \bigl(2^{N(H) \, \cap \, \operatorname{ext}(H\mid B)},\ \subseteq\bigr).
\]

\noindent Even though $F(H)$ is defined as the pair containing the set and the order, sometimes we use $F(H)$ to denote the corresponding set. 

\end{mydefinition}

\noindent
We define the set of $(H,X)$ pairs an SNG can return.

\begin{mydefinition}[Set of candidate $(H,X)$ pairs]
\label{def:groth}
Given the base-poset $\mathcal H$ and, for each $H$, the associated
collection of valid instance-sets, we define
\[
\mathcal F \;=\; \{\, (H,X) \mid H\in\mathcal H,\ X\subseteq N(H) \cap \operatorname{ext}(H\mid B)\,\}.
\]
\end{mydefinition}%

\noindent
The schematic depiction of the base
poset and the associated fibre-posets is shown in Fig.~\ref{fig:fiber},
along with the corresponding ${\mathcal F}$.


\begin{figure}[!htb]
\centering
\begin{subfigure}{0.48\textwidth}
\centerline{\includegraphics[height=0.225\textheight]{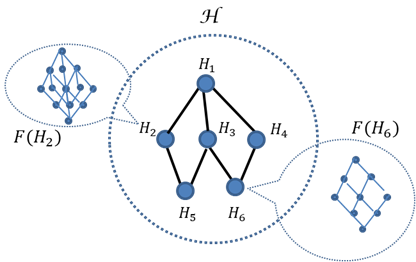}}
  \caption{}
\end{subfigure}%
\begin{subfigure}{0.48\textwidth}
\centerline{\includegraphics[height=0.225\textheight]{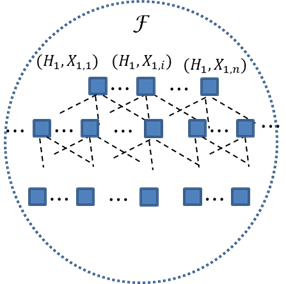}}
\caption{}
\end{subfigure}
\caption{(a) Posets indexed by elements of a base poset ${\cal H}$. Each
    element $H$ of the base poset is associated with a fibre-poset $F(H)$.
    (b) Pairs $(H,X)$ with $H \in \mathcal H$ and $X \in F(H)$ form
    the set $\mathcal F$ of candidate hypothesis--instance-set combinations
    an SNG returns. The set ${\mathcal F}$ is shown as a partially-ordered
    set: see text for details.}
\label{fig:fiber}
\end{figure}

Elements of ${\mathcal F}$ are the candidate neural--symbolic pairs
that an SNG can return. We note the following about ${\mathcal F}$:

\begin{myremark}
    \item The set ${\mathcal F}$ is partially ordered. It is straightforward
        to show that the relation $\geq_{\mathcal F}$ defined as
        $(H_1,X_1) \geq_{\mathcal F} (H_2,X_2)$ iff
        $(H_1 \geq_{\mathcal H} H_2)$ and $(X_1 \supseteq X_2)$ induces
        a partial ordering on ${\cal F}$.
    \item In the worst-case,
        $|{\mathcal F}| \leq |{\mathcal H}| \cdot  2^{|{\cal U}|}$.
        However, the requirement that $F(H) \subseteq 2^{\operatorname{ext}(H|B)}$
        can result in a reduction in the actual size of
        ${\cal F}$. However this will only reduce the upper
        bound to $|{\mathcal H}| \cdot  2^n$ where
        $n = \max_H 2^{|\operatorname{ext}(H|B)|}$.
\end{myremark}

These observations suggest that it can be impractical to search 
${\mathcal F}$ for a suitable $(H,X)$ pair. The implementation
we use for real applications below uses
$\geq_{\mathcal H}$ and a ``goodness'' score associated
with $(H,X)$ pairs to drive a a greedy search. This reduces the size of
the search-space substantially,
at the cost of returning a locally optimal $(H,X)$ pair.

\paragraph{Note on Probabilistic Extension.}
We note two sources of uncertainty in hybrid neurosymbolic
systems:
(a) The symbolic hypotheses  may not
all be equally likely, given data and background knowledge;
(b) The neural generator is usually stochastic and the support-sets
are samples to which we can attach a probability (of obtaining the sample
under the sampling distribution conditioned by $H$ and $B$). 
Accounting for these uncertainties will require an
extension to a probabilistic semantics that specifies
a probability distribution over the elements of ${\mathcal F}$.
For the present, the scoring function we use in the implementation to weigh
$(H,X)$ pairs can be viewed as an implicit specification of
an underlying distribution.

In addition to the sets defined so far, we may know beforehand some elements of
${\cal U}$ for which $\Phi(\cdot)$ is true or false.
These are provided as pairs of subsets
from $2^{\mathcal U} \times 2^{\mathcal U}$, which we denote by $\mathcal E$. We can now specify a class of functions we
call \textit{symbolic neural generators}.

\begin{mydefinition}[Symbolic Neural Generator]
 \label{def:sng}
 Let $\mathcal U$, $\mathcal B$, $\mathcal E$, $\mathcal H$ be as above, with
 $(\mathcal H,\ge_{\mathcal H})$ partially ordered by extensions
 (Defn.~\ref{def:pohyp}) and $\mathcal F$ be the
 set of $(H,X)$ pairs as before.
 A Symbolic Neural Generator is any function
 $SNG: {\cal E} \times {\cal B} \to {\cal F}$.
\end{mydefinition}

\subsection{Implementing SNGs}
\label{sec:sngimp}

As defined, an SNG can be implemented by a function
that returns an $(H,X)$ pair, given as inputs a
sample of data and background knowledge.
In designing an implementation, we are required
to address 3 questions:
(i) How do we identify the $H$'s;
(ii) How do obtain the $X$'s; and
(iii) How do we search for a ``good'' $(H,X)$ pairs?
    
Of these, Question (ii) is the most pressing, since it is
here that neural and symbolic components interact.\footnote{Question (i) 
is concerned with constructing
symbolic hypothesis, given examples and background knowledge.
Techniques developed in ILP are especially good at this.
Question (iii) is about searching graphs efficiently, for
which several alternatives are known.}

Procedure \ref{alg:nstar} is a general-purpose
implementation for addressing Question (ii) using
a large language model (LLM) as a generator of
instances from ${\cal U}$. 
\Gen employs rejection sampling with contextual updates
to ensure that elements in $X$ are in $\operatorname{ext}(H|B)$.\footnote{
We do not have to explicitly construct $\operatorname{ext}(H|B)$ for this.
Instead, we include $x$ in $S$ iff $B \wedge H \models \Sigma(x)$. This
check can be performed by a model-checker (or in some restricted cases,
by an inference engine). 
Also we
do not accept any instance $x$ that vacuously satisfies the requirement that
$B \wedge H \models \Sigma(x)$.
}
It  is reasonably common in the use of language models to include some
data in the initial context \citep{brown2020language}. For simplicity, we will
assume these are provided in $E$.

\begin{algorithm}[!htb]
    \caption{\Gen: An LLM-based implementation of the neural generator with
        rejection-sampling.}
    \label{alg:nstar}
    \textbf{Input}: $L$: an LLM capable of generating instances in a set
        conditional on the hypothesis $H$ (see below);
        $E$: a subset of known data instances;
        $B$: background knowledge as provided to symbolic learner;
        $H$: a symbolic hypothesis for $E$ given $B$;
        $n$: an upper-bound on the number of iterations; and
        $s$: an upper-bound on the number of samples to be drawn using the LLM; \\
    \textbf{Output}: a pair $(w,X)$ where $w \in [0,1]$ and $X \in F(H)$ $(= 2^{N(H) \cap \operatorname{ext}(H|B)})$. 
    
    \begin{algorithmic}[1]
    \STATE Let $C_0$ be the initial context containing a description
        of $H$ and $E$\; \label{step:start}
    \STATE $M_0 = \emptyset$\;
    \STATE $w_0 = 0$\;
    \STATE $i := 1$\;
    \WHILE{$i \leq n$} \label{step:startloop}
\STATE Let $P_i$ be a prompt to generate $x$ s.t. $x$ in ${\cal U}$, and it completes the sentence $true: x$ given context $C_{i-1}$\;
        \STATE $S_i = \mathit{Sample}(s,L,P_i)$\; {\gray{//sample at most $s$ instances from ${\cal U}$ using the LLM}}
        \STATE $D_i := \{(l,x): x \in S_i$ and $l := (x \in \operatorname{ext}(H|B)) \}$\; {\gray{//$l$ is either True or False}}\label{step:check}
        \STATE Let context $C_i$ be $C_{i-1}$ updated with $\mathit{true}:x$ for $(\mathit{true},x) \in D_i$ and $\mathit{false}:x$ for $(\mathit{false},x)$ in $D_i$\; 
        \STATE $M_i = \{x: (\mathit{true},x) \in D_i\} \cup M_{i-1}$\; \label{step:update}
        \STATE $i := i + 1$\;
    \ENDWHILE \label{step:endloop}
        \STATE $w_n := \frac{\mid M_{n}\mid}{(s \times n)}$
    \RETURN $(w_n,M_n)$
    \end{algorithmic}
\end{algorithm}

\noindent
 
The following proposition is evident enough from Step \ref{step:check} in
Procedure \ref{alg:nstar}, but is nevertheless worth reinforcing and the proof is given in Appendix \ref{app:correctnessGen}:

\begin{myproposition}[Correctness of \Gen]
\label{prop:gencorrect}
The set $M_n$ returned by Procedure \ref{alg:nstar}  is an
element of $F(H)$ and $w_n \in [0,1]$. 
\end{myproposition}

A clarification is needed on the role of the $w_i$. On each iteration,
this is the fraction of generated instances that are accepted as being within
$\operatorname{ext}(H|B)$. This is intended as a measure of `alignment' of
the neural-generator with the symbolic hypothesis. This becomes useful
later, where \Gen is used as part of an SNG implementation that interleaves
symbolic hypothesis search with neural-generation.

\subsection{Demonstration of a simple SNG}
\label{sec:synth}
A simple SNG results from adopting the following
``decomposition'' strategy: (a) First search the base poset to find good
symbolic hypothesis $H$ (this can be done with an ILP engine, for
example); and (b) Use the $H$ obtained in (a)  and Procedure
\Gen to obtain a sample of instances from $\operatorname{ext}(H|B)$.
Let us call this a \textit{chain-processing} SNG, based on the sub-categorisation
in \citep{hilario2013overview}.

We demonstrate a chain-processing  SNG using the chess endgame consisting of
the White King, White Rook and Black King. Two problems for this endgame have been studied in the
symbolic learning literature: learning rules for detecting illegal positions; and
identifying the minimum number of moves to a win for the Rook's side assuming it is
Black's turn to move. The first problem is the more widely studied, but is less
interesting since it follows almost directly from the rules of the game. We will use the second problem, and specifically on identifying positions that White is 0 moves
away from a win (one such example is shown in
Fig.~\ref{fig:gcws}(a)). That is, with Black to move it is checkmate and therefore won-for-white
(WFW). These positions constitute about 0.1\% of the entire dataset,
and is therefore a rare event. The goal of the SNG will be to:
(a) identify a symbolic hypothesis for WFW; and (b) generate
instances for this rare event. A symbolic
hypothesis has long been available in the literature, constructed
by the GCWS extension to the ILP engine Golem
\citep{bain:gcws} (see Fig.~\ref{fig:gcws}(b)). 

\begin{figure}[!htb]
\begin{subfigure}{0.4\textwidth}
    \centering
    \includegraphics[width=0.5\linewidth]{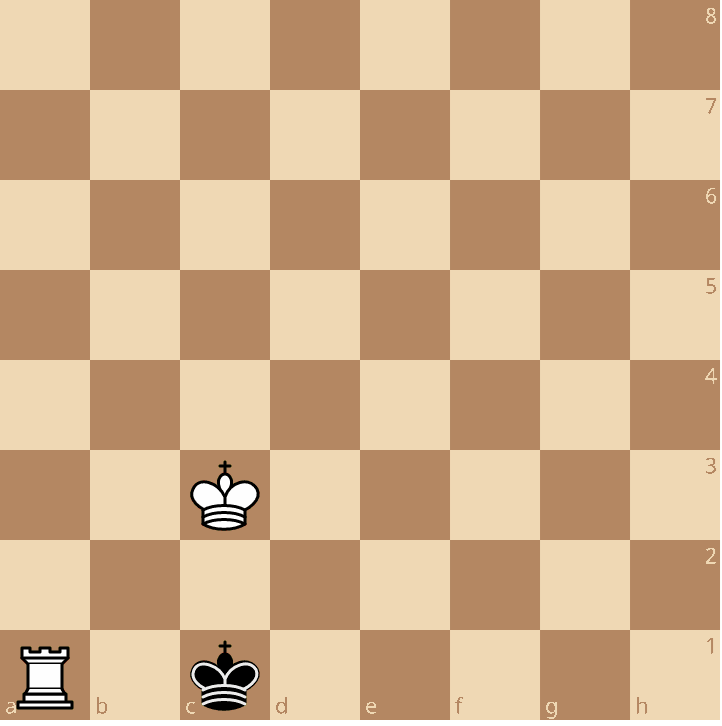}
    \caption{}
\end{subfigure}%
\hfill
\begin{subfigure}{0.4\textwidth}
    \centering
    {\scriptsize{
    \begin{tabbing}
        \= krk(A,B,\= \kill
        \> $H:$ \\
        \> $\Sigma$((WKF,WKR,WRF,WRR,BKF,BKR)) :- \\
        \> \> depth\_of\_win(0,WKF,WKR,WRF,WRR,BKF,BKR). \\[6pt]
        \> depth\_of\_win(0, c, 2, a, A, a, 2) :- not(ab3(0, c, 2, a, A, a, 2)).\\
        \> depth\_of\_win(0, c, A, a, B, a, 1) :- not(ab2(0, c, A, a, B, a, 1)).\\
        \> depth\_of\_win(0, A, 3, B, 1, A, 1) :- not(ab1(0, A, 3, B, 1, A, 1)).\\
        \> ab1(0, A, 3, B, 1, A, 1) :- diff(A, B, d1).\\
        \> ab2(0, c, A, a, 2, a, 1).\\
        \> ab3(0, c, 2, a, A, a, 2) :- diff(2, A, d1).\\
    \end{tabbing}
    }}
    \caption{}
\end{subfigure}
\caption{(a) A position that is ``won-for white'' (WFW) with ``black-to-move''. 
    Here depth-of-win is zero, i.e., checkmate. There are 27 such positions
    out of a total of 28,056; (b) A symbolic hypothesis obtained using ILP
    (adapted from \citep{bain:gcws}) (rewritten using $\Sigma$ as required).
    The 6-tuple $(\mathit{WKF},\mathit{WKR},\mathit{WRF},\mathit{WRR},\mathit{BKF},\mathit{BKR})$
    encodes board coordinates: $\mathit{WKF}$, $\mathit{WKR}$ are the file and rank of the
    White King; $\mathit{WRF}$, $\mathit{WRR}$ of the White Rook; and $\mathit{BKF}$, $\mathit{BKR}$
    of the Black King. For example $(c,3,a,1,c,1)$ is the coordinates of the position in the figure.
    In the context of this paper ${\cal U}$ consists of 6-tuples
    representing positions of the 3 pieces.
    The description is a Prolog-like syntax:
    variables start with upper-case, ``:-'' stands for $\leftarrow$, and ``not'' should
    be read as ``not provable''. The ``diff/3'' predicate is defined in the background
    knowledge and encodes file or rank differences. The ``ab'' predicates are new
    relations invented by the ILP system as it attempts to find a logical
    description for the depth-0 data instances. For example, in the position shown in (a), the invented ``ab1'' predicate ensures the white rook cannot immediately be taken by the black king. $\Sigma$ should be
    understood as ``WFW''.}
\label{fig:gcws}
\end{figure}

Below we show a summary of iterations of \Gen when
the language model used is GPT-4o, and the maximum sample-size is set to 30.
In the figure, ``Without Symbolic'' denotes there is no symbolic theory available; and
``With Symbolic'' refers to using the theory learned
by the ILP engine (Fig.~\ref{fig:gcws}(b)). ``0-shot'' means no
explicit examples of the concept WFW are provided within the hypothesis;
and ``5-shot'' refers to providing 5 (positive) instances of WFW.
From the results in Table~\ref{fig:chess_res} it is evident that the use of the symbolic model enables a progressive improvement in the
conditional generation by the LLM. Given the small number of positive
instances overall for the concept,
it is unsurprising that LLM does not obtain any positive
instances in the 0-shot samples  without a symbolic theory.\footnote{
The results are with a sample-size of $30$ on each iteration.
For a uniform sampler, the probability of a positive instance in any single
random draw is
approximately $27/27000 \approx 0.001$. The chance of not obtaining at least
1 positive instance in a sample of $30$ is $\approx 1-(0.999)^{30} = 0.03$.
It is unclear if this applies directly to language models, which have
access to significant amounts of prior chess knowledge.
}

\begin{table}[!htb]
    \centering
   
    \begin{tabular}{|c|c|c|c|c|}
        \hline
        \multirow{2}{*}{Iteration ($i$)} & \multicolumn{4}{|c|}{$|M_i|$} \\
          & \multicolumn{2}{c|}{Without Symbolic} & \multicolumn{2}{c|}{With Symbolic} \\ \cline{2-5}
          & 0-shot  & 5-shot & 0-shot & 5-shot \\ \hline
        1 & 0 & 8   & 17 & 23 \\
        2 & 0 & 23  & 30 & 30 \\
        3 & 0 & 25  & 30 & 30 \\
        4 & 0 & 25  & 30 & 30 \\
        5 & 0 & 27  & 30 & 30 \\
        \hline
    \end{tabular}
    \caption{Instances of WFW generated on each iteration of \Gen.
   ``Without Symbolic'' represents the baseline of the LLM generating instances
    without any symbolic theory as part of the initial context or for verification
    (that is, $H = \emptyset$ in \Gen). ``With Symbolic'' provides the WFW theory in Fig.\ref{fig:gcws}(b). ``0-shot'' means no examples are provided in $E$ (and
    therefore are not part of the initial context), and ``5-shot'' means 5 WFR
    positions are provided in $E$.}

    \label{fig:chess_res}
\end{table}

What may surprise the reader, however, is that the generator appears to have
generated 3 more instances -- 30, instead of 27 -- consistent with the symbolic theory.
How is this possible, given that the symbolic description is stated in \citep{bain:gcws}
as a complete and correct recogniser for WFW. Closer study of the
description in \citep{bain:gcws} reveals that the symbolic description
is only intended to be a complete and correct description of \textit{legal}
positions. The 3 additional positions generated are in fact all illegal
positions. Thus, the neural generator
has generated unexpected instances that are  consistent with the theory. This
can be taken in 2 ways:
(a) the verification step is only as good as the
theory it is being verified against; and
(b) the SNG can identify unanticipated instances that are
consistent with the symbolic theory. This second aspect can be
seen as a positive feature of the approach, especially in real-world problems
such as the ones we consider later in the paper.

The reader could question the utility of the  symbolic theory, on grounds that the 
LLM's performs well enough simply given a few initial examples (without symbolic theory,
5-shot). For this well-known synthetic dataset, this is indeed the case. However, even here
we have still not generated all WFW instances after 5 iterations. For problems where
the data instances are rare, missing even a few makes a difference (as in the real-world
problems we consider next).

\section{Application: SNGs for Lead Discovery in Drug-Design}
\label{sec:appl}

We now turn to the real-world problem of generating small
molecules subject to constraints imposed by a protein-target.
For this, we will implement a \textit{sub-processing} SNG (\citep{hilario2013overview})
which interleaves neural
generation within the search for a symbolic hypothesis.\footnote{
 Unlike the chess problem, we no longer have a symbolic hypothesis.
}
For reasons
of space, we do not provide a review of the role of small-molecule
identification for early-stage drug design, and refer the reader to 
the literature (see, for example,~\citep{doytchnova:drug}). More
specifically, and closely related to SNGs is the work in
\citep{ross:breast}. There, an LLM is used both
to generate hypotheses and small molecules, and a robot-scientist
is used to test the molecules proposed.\footnote{We conjecture that the 
use of
the symbolic learner, as is done here, may result in more reliable hypotheses
and require fewer experiments to be conducted by the robot.}

The starting point of any SNG is the
the identification of the base poset $(\mathcal H,\ge_{\mathcal H})$.
Here, we informally describe the hypothesis space for describing small
molecules: a formal treatment
is in Appendix \ref{app:alg}. Simply put, hypotheses for the
problems considered here will encode interval-constraints on properties -- which we will call \textit{factors} -- of molecules. An assignment of interval-values to factors will be
called an \textit{experiment}. Factors together with an experiment will
be called a \textit{factor-specification}. Hypotheses encode factor-specifications.

\begin{myexample}[Factors, Experiments, Hypotheses]
\label{ex:factors}
Consider the factor-specification:\\ $((\mathit{MolWt}$, $\mathit{SynthSteps},\mathit{Affinity}), ([200,800], [4,8],[8,10]))$
specifies we have 3 factors -- molecular weight, number of synthesis steps, and
estimated binding affinity to the target -- and an experiment that assigns
$\mathit{MolWt}$ to the interval $[200,800]$, $\mathit{SynthSteps}$ to
the interval $[4,8]$ and $\mathit{Affinity}$ to the
interval $[8,10]$. Each such experiment is treated
as a hypothesis -- which we denote $\mathit{Hypothesis}((f_1,\ldots,f_n),(i_1,\ldots,i_n))$ -- and
represented as a clause. Thus
$\mathit{Hypothesis}$ $((\mathit{MolWt},\mathit{SynthSteps,}\mathit{Affinity}),([200,800],[4,8],[8,10]))$ is the clause
(here, ${\cal U}$ is the set of small molecules).
\begin{center}
{\footnotesize
\begin{tabbing}
\hspace*{2.5cm}\=P(x) $\leftarrow$ \= \kill
\>$\forall x \in \mathcal{U}~\Sigma(x) \leftarrow$ \\
\>\> $\mathit{MolWt}(x) \in [200,800] ~\wedge$ \\
\>\> $\mathit{SynthSteps}(x) \in [4,8] ~\wedge $ \\
\>\> $\mathit{Affinity}(x) \in [8,10]$
\end{tabbing}
}
\end{center} 
\end{myexample}
\noindent 
Hypotheses are thus simple axis-parallel hyper-rectangles
(akin to a categorical version of the probabilistic
box-embedding introduced by \cite{vilnis2018probabilistic}: a probability
estimate will be added later).
We will assume that background knowledge $B$ will specify how to
compute the values of the factors.\footnote{Some computations, like
$\mathit{Affinity}(\cdot)$, can be quite involved, requiring knowledge of
the target-site, and the use of a molecular modelling program for estimating
binding affinity of the molecule to this site.}

We note the hypothesis
specifies constraints on the factor-values for acceptable
molecules.
In practice domain experts identify a small set of chemically
meaningful factors $n$. The cost of \GM is dominated by the sampling budget.
The search constructs $n$-dimensional bounding hyper-rectangle rather
than examining $2^n$ subsets of factors (and corresponding hyper-rectangles). 
The greedy search procedure thus constructs at most $s \times k$
hyper-rectangles.

Identifying a suitable SNG requires searching through potential $(H,X)$ pairs
(here $X$ will be a set of molecules within $\operatorname{ext}(H|B)$).
An informed search naturally requires the specification of a score
for $(H,X)$ pairs. We use the following scoring function.

\begin{mydefinition}[Scoring $(H,X)$ pairs]
\label{def:probext}

Let $\mathcal H$ be a hypothesis class, $B$ background knowledge, and
$E$ denote initial data.
Given a $H\in\mathcal H$ and an associated support-set
$X\subseteq N(H) \cap\operatorname{ext}(H\mid B)$. Let 
$w(H,X) \in [0,1]$ be an empirical weight summarising the alignment of
$X$ with $\operatorname{ext}(H\mid B)$,
and $P(H\mid E,B)$ be the posterior  probability
of the hypothesis $H$. Let $P(H \mid E, B)$ be approximated by the
Bayesian score $Q(H,E,B)$ which is a closed-form
estimator of $P(H|E,B)$ (see \citet{mccreath1998lime}
and Appendix \ref{app:Qscore}).
We associate the following combined weight with $(H,X)$:
\[
W(H,X)\;=\;Q(H,E,B)\,\cdot\,w(H,X).
\]
\end{mydefinition}

In our experiments we use Procedure \ref{alg:genmol}, which
 performs a greedy search for an $(H,X)$ pair that maximises a
 slightly simplified form of $W(H,X)$.\footnote{
 $w(H,X)$ is replaced by an indicator function that is $1$ if
 $w(H,X) > 0$. In effect, this rejects hypotheses for which the neural
 generator is unable to generate consistent instances. From the
 rest, the search returns the hypothesis with highest
 Q-score, and therefore highest posterior probability.}

\begin{algorithm}[!htb]
    \caption{\GM: implementing a greedy search over the space of 
        symbolic descriptions and consistent sets of generated molecules}
    \label{alg:genmol}
    \textbf{Input}:
        $L$: an LLM;
        $E$: a  sample of labelled
              and unlabelled molecules; 
        $B$: background knowledge, that includes
        $(F,\mathbf{\Theta})$: a factor-specification with
            $F = (f_1,\ldots,f_n)$ and
            $\mathbf{\Theta} = ([\theta_1^-,\theta_1^+],\ldots,[\theta_n^-,\theta_n^+])$;
        $s$: an upper-bound on the number of samples in the search; 
        $k$: an upper-bound on the number of steps in the search;
        $l$: an upper-bound on the number of LLM iterations;
        $m$: an upper-bound on the number of samples to be drawn by the LLM' and
        $\theta$: a lower bound on the weight of an acceptable hypothesis\\
    \textbf{Output}: $(H,X) \in \mathcal F$ where
        $H \in (\mathcal H,\ge_{\mathcal H})$
        is a symbolic hypothesis; and
        $X \in F(H)$ a set of molecules.

    \begin{algorithmic}[1]
    
    \STATE $\mathbf{e}_0 = \mathbf{\Theta}$\;
    \STATE $H_0 = Hypothesis(F,\mathbf{e}_0)$\;
    \STATE $(w,M_0)$ = \Gen$(L,B,H_0,l,m)$\;  {\gray{//generate molecules in outermost hyper-rectangle}}
    \STATE $q_0 = Q(H_0,E,B)$\;   {\gray{//see Appendix \ref{app:Qscore} for Bayesian score}}
    \STATE $w_0 = q_0 \times \mathbf{1}(w > 0)$\;
    \STATE $i = 1$\;
    \STATE $Done = ((w_0 < \theta) \vee (i > k))$\;
    \WHILE{$\neg Done$}
        \STATE Let $E_i$ be a random sample of $s$ interval-vectors properly subsumed by $\mathbf{e}_{i-1}$\label{step:esample}\;
        \STATE Let $S = \{(w',\mathbf{e},H,M): \mathbf{e} \in E_{k},~ H = Hypothesis(F,\mathbf{e}),~ q = Q(H,E,B), (w,M) =$ \Gen$(L,E,B,H,l,m), w' = q \times {\mathbf 1}(w > 0)\}$\; \label{step:estimate}
        \STATE Let $(W_i,\mathbf{e}_i,H_i,M_i) = \underset{(w,\mathbf{e},H,M) \in S}{\text{argmax}} ~w$\;\label{step:select}
        \STATE $Done = ((w_i < \theta) \vee (W_i < W_{i-1}))$\;
        \STATE $i := i + 1$\;
    \ENDWHILE
    \RETURN $(H_{i-1},M_{i-1})$
    \end{algorithmic}
\end{algorithm}

\noindent
We note that \GM interleaves a symbolic learner
with a neural reasoner \Gen. It therefore
exemplifies the \textit{sub-processing} form of integration
identified in \citep{hilario2013overview}.
We clarify the working of two aspects of the search by example.
First, we highlight the hypotheses considered:

\begin{myexample}[Nested Rectangles]
Suppose \GM is attempting to find a hypothesis given 2 factors
$F = (\mathit{Affinity},\mathit{SynthesisSteps})$, with $\Theta = \mathbf{e}_0 = ([5,10],[4,8])$. Then
\GM starts with the constraint
$(\mathit{Affinity} \in [5,10]) \wedge (SynthesisSteps \in [4,8])$, which
corresponds to a rectangle $R_0$ in the Cartesian-space with $\mathit{Affinity}$
and $SynthesisSteps$. Let the $Q$-value of the corresponding hypothesis be $q_0$.
\SH then randomly samples rectangles contained within $R_0$. Suppose the
rectangle $R_1$, defined by  $\mathbf{e}_1 = ([6,10],[4,6])$, and the corresponding hypothesis
has the highest $Q$-value ($q_1$)  of all the rectangles sampled. \SH then iterates
by sampling within $R_1$. The search procedure therefore identifies a sequence of
nested rectangles.   
\end{myexample}

Secondly, we note that \GM is a randomised procedure. Sampling is done in Step \ref{step:esample}
of \GM. Here there are several options: the easiest to sample each
dimension of the bounding hyper-rectangle
independently, using a uniform distribution. Better sampling procedures exist
(for example, Latin Hyper-rectangle Sampling \citep{mckay2000comparison}, DIRECT \citep{jones2001direct},
Bayesian sub-region sampling \citep{skilling2004nested} and so on). Alternatives to the simple greedy strategy of picking the best-scoring hyper-rectangle in Step~\ref{step:select} 
also clearly possible, by drawing an experiment from the distribution of $w$-scores.

\begin{myexample}[Sampling (Hyper-)Rectangles]
Suppose \GM is attempting to sample a rectangle bounded by $[x_1,x_2]$ and $[y_1,y_2]$. 
Examples of some strategies for sampling sub-rectangles are:

\noindent
{\bf{Uniform orthogonal sampling.}}
(a) Select a pair of points $a = U(x_1,x_2)$ and $b = U(x_1,x_2)$ s.t. $x_1 < a < b < x_2$;
(b) Select a pair of points $c = U(y_1,y_2)$ and $d = U(y_1,y_2)$ s.t. $y_1 < c < d < y_2$; and
(c) The new rectangle is bounded by $[a,b]$ and $[c,d]$.

\noindent
{\bf{Sampling with fixed upper- or lower-bounds.}}
(a) Select a point $a = U(x_1.x_2)$ s.t. $x_1 < a < x_2$;
(b) Select a  point $d = U(y_1,y_2)$ s.t. $y_1 < d < y_2$; and
(c) The new rectangle is bounded by $[a,x_2]$ and $[y_1,d]$. 
\end{myexample}

We treat the choice of sampling method
as an application-specific detail. The
following properties will, however, hold for Procedure \GM.
The entire randomised search over hyper-rectangles can
be performed by an ILP engine like Aleph, with a user-specified
\emph{refinement operator} that returns samples of sub-rectangles,
and a user-specified scoring function. The symbolic hypothesis
construction within \GM can thus be thought
of as being performed by a specialised, more efficient, ILP learner.

\subsubsection*{Properties of \GM}
\noindent
Irrespective of the sampling strategy used, the following proposition
holds for \GM and the proof is
given in the appendix \ref{app:proofcmimpl}:

\begin{myproposition}
\label{prop:cmimpl}
Let $(F,\cdot)$ be a factor-specification, $B$ denote background knowledge.
    Let $\mathbf{e}_{i}$ ($1 \leq i \leq k$) be an experiment selected by
     \GM on the $i^\text{th}$ iteration s.t. $\mathbf{e}_i$ is contained by $\mathbf{e}_{i-1}$.
Let $H_i = Hypothesis(F,\mathbf{e}_i)$, and $H_{i-1} = Hypothesis(F,\mathbf{e}_{i-1})$.
Then $H_{i-1} \models H_i$.
\end{myproposition}

\noindent
It is straightforward
to see that the hypotheses examined by \GM are from the base poset $({\cal H},\geq_{\cal H})$.

\begin{myremark}[$H_{i-1} \geq_{\cal H} H_i$]
\label{rem:searchorder}
Let $H_{i-1}, H_i \in {\cal H}$ be hypotheses constructed by \GM on iterations $i-1, i$.
Let $\geq_{\cal H}$ be as defined in Defn.~\ref{def:pohyp}.
Then $H_{i-1} \geq_{\cal H} H_i$.
This follows from Prop. \ref{prop:cmimpl}. If $H_{i-1} \models H_i$ then
$B \wedge H_{i-1} \models B \wedge H_i$. It follows that
$\operatorname{ext}(H_{i-1}|B) \supseteq \operatorname{ext}(H_i|B)$, and $H_{i-1} \geq_{\cal H} H_i$.
\end{myremark}

\begin{myremark}[\textbf{\GM implements an $SNG$}]
\label{rem:sng}
We note first that the symbolic hypotheses examined by \GM 
    are from ${\cal H}$.
    Let $(H_i,M_i)$ be the hypothesis, generated set of molecules and
    weight in \GM
    on iteration $i$.
    By construction $M_i$ is obtained using \Gen with hypothesis $H_i \in {\cal H}$.
    By Prop.~\ref{prop:gencorrect}
    $M_i \subseteq \operatorname{ext}(H_i|B)$.
    $(H_i,M_i) \in {\cal F}$ and
    \GM can be taken to be an implementation of an SNG function defined in
    Defn.~\ref{def:sng} (with some additional bounds on sample sizes, \textit{etc\/}.)
\end{myremark}

\subsubsection*{Limitations of \GM} 
We draw the reader's attention to the following limitations of
\GM:
\begin{itemize}
    \item \GM uses the base-poset ordering $\geq_{\mathcal H}$ to
    search the hypothesis space and a separate sampling step to populate
    each $X$. It does not maintain a joint best-first ordering over
    $(H,X)$ pairs; consequently the search is greedy with respect to
    hypothesis quality and does not back-track on the choice of $X$
    given a fixed $H$.
    
\item \GM assigns a 0 or 1 value for generator efficiency in Step \ref{step:estimate}        
    (through the use of the indicator function $\mathbf 1(\cdot)$). This results in
    over-zealous pruning of the hypothesis space. A better estimate would
    be to use a measure based on the proportion $Q \times w$ as described
    in Sec.\ref{sec:sem}.
\item The search-strategy performs a greedy selection of a single
    element in Step \ref{step:select}. Again, while this reduces the size of the search-space,
    It is likely that a best-first strategy employing
    a priority-queue would be an improvement.
\item \GM inherits the inefficiency of \Gen, which adopts rejection-sampling
    to ensure that data instances returned satisfy the the symbolic hypothesis.
\end{itemize}

\subsection{Case Study 1: Well-Understood Targets with Many Ligands}
\label{sec:case1}

We evaluate the performance of SNG in the controlled setting
examined in \citep{brahmavar2024generating}, with a known target-site, a
large number of known inhibitors and non-inhibitors,
and a single factor to be optimised (estimated binding affinity to the
target-site). 

\subsubsection*{Problems}
{\bf{Kinase Inhibitors.}} We conduct our controlled
evaluations on 2 well-studied kinase inhibitors: (a) JAK2, with 4100 molecules
provided with labels (3700 active); and (b) DRD2 (4070 molecules
with labels, of which 3670 are active). These datasets are from the ChEMBL database \citep{gaulton2012chembl}, which are selected based on their $IC_{50}$ values and docking scores.
JAK2 inhibitors are drugs that inhibit the activity of the Janus kinase 2 (JAK2) enzyme, in
turn affecting signalling pathways, especially to the cell nucleus.
These pathways are critical for various immune response reactions and are used
to develop drugs for autoimmune disorders like ulcerative colitis and rheumatoid arthritis.
DRD2 (dopamine D2) inhibitors are drugs that block dopamine's ability to activate the
DRD2 receptors. This reduces dopamine signalling, and is used to treat psychological disorders
like schizophrenia.\\

\subsubsection*{Background Knowledge}

We distinguish the following components of background
knowledge:
\begin{enumerate}[(A)]
    \item Specialists' knowledge. We require biological knowledge
        of the target site, or a proxy for the target size. 
        We also need chemical knowledge of the relevant
        factors, the range of their values, and information (if any) of whether the
        factors need to be maximised or minimised. For this case
        study, we will use only 1 factor (estimated binding affinity).
    \item Factor-functions. These are definitions for
        computing the factors (like \textit{Affinity},
        \textit{MolWt} {\it {etc}\/}.) for molecules. Usually
        this will also include procedures provided by some
        molecular modelling software (like RDKit).
\end{enumerate}

\subsubsection*{Algorithms and Machines}
All the experiments are conducted using a Linux (Ubuntu) based workstation with
64 GB of main memory, a 12-core (dual) AMD Ryzen Threadripper processor, and
NVIDIA RTX 4500 Ada Generation GPU (24 GB). 
All the implementations are in Python3. We use \texttt{OpenAI} library (version 1.52.1)
for sampling molecules from \texttt{GPT-4o} \citep{achiam2023gpt}, with temperature $0.7$. 
We use \texttt{RDKit} (version 2024.09.3) \citep{bento2020open} for computing
molecular properties and \texttt{GNINA} (version 1.3) \citep{mcnutt2021gnina}
for computing docking scores (binding affinities) of molecules.
Additional details of the experimental setup can be found in Appendix~\ref{app:expt}. We have used PubChem Sketcher V2.4 for drawing the
2-D structures of the molecules shown in this paper.

\subsubsection*{Experimental Method}

Our method for experiments is straightforward and follows these steps:

\begin{enumerate}
    \item Identify factor-set $F$ and other bounds.
    \item Obtain data instances consisting of positive, negative and unlabelled examples.
        Any positive example is taken to be feasible and any negative example is taken
        to be infeasible.
    \item Using the background knowledge $B$ described in Sec.~\ref{sec:case1}, $D$, $F$
            and other bounds:
        \begin{enumerate}
            \item Obtain a set of molecules using \GM;
            \item Assess the quality of the molecules generated.
        \end{enumerate}
\end{enumerate}

\noindent
We refer the reader to Appendix \ref{app:expt} for additional details related
to the method.

\subsubsection*{Results}

Table~\ref{tab:valresults1} shows a comparison of SNG against the
results tabulated for LMLF++ in \citep{brahmavar2024generating}.
LMLF++ was the best performing variant in that paper, and was
substantially better than previous benchmarks set by the use
of a VAE-GNN combination and reinforcement-learning based methods.
It is evident from the tabulation that SNG performs
at least as well as LMLF++.\footnote{
\blue{Neural Markov Logic Networks~\citep{marra2021neural} are a neurosymbolic relational generative model demonstrated on molecule generation, and would be a natural point of comparison. We attempted this using the authors' official implementation. The reported NMLN molecule experiments used $1073$ molecules restricted to exactly $8$ heavy atoms, whereas our molecules are drug-like ($13$ to $53$ heavy atoms, median $31$). The NMLN formulation grounds all predicates over a fixed domain of $N$ constants per molecule, and the unary one-of-$N$ constraint requires exactly $N$ heavy atoms per training example; setting $N=32$ retained only about $350$ of about $4000$ molecules. 
Even with reduced potential arity ($k=3$), training on a $24$\,GB GPU completed only ${\sim}3000$ of the default $200{,}000$ iterations after $48$ hours of GPU-saturated runtime, well short of convergence. We therefore do not report a direct comparison: NMLN was developed and evaluated in a small-fragment regime, and scaling it to drug-like sizes lies outside the scope of the current work.}}

It is also helpful if a
lead generator proposes novel molecules. For the
JAK problems, this is not easy, since the number of
known inhibitors is very large. Nevertheless, Table~\ref{tab:valresults2} 
suggests that the molecules generated by \GM may still
be quite novel (this is due to the prompt used for the LLM that
attempts to generate molecules not in any known chemical
database). 
\blue{The Tanimoto coefficient ($TC$) is a widely used
fingerprint-based similarity between two molecules, ranging from 0
(no shared bit-pattern) to 1 (identical fingerprints); a threshold of
$TC \ge 0.4$ is commonly considered ``similar'' in
cheminformatics. The values in Table~\ref{tab:valresults2} (0.13--0.14)
therefore indicate that the molecules generated are structurally quite
distinct from the known inhibitors, which is a desirable property for a
lead-discovery tool.}

\blue{%
\begin{table}[!htb]
\centering
{\small
\begin{tabular}{|c|c|c|c|c|}
\hline
Problem & Known Inhibitors & LMLF++      & VAE-GNN     & \GM (Ours)  \\ \hline
JAK2    & 7.26 (0.64)      & 7.74 (0.30) & 6.53 (1.18) & 8.59 (0.23) \\ \hline
DRD2    & 6.86 (0.83)      & 7.66 (0.29) & -           & 7.53 (0.47) \\ \hline
\end{tabular}}
\caption{Statistics of binding affinities (the higher the better)
    for molecules obtained from
    \GM on benchmark datasets. The entries represent the mean values,
    with standard deviations shown in parentheses. We compare against
    recent results using LMLF++ \citep{brahmavar2024generating} and
    prior results using a VAE-GNN model \citep{dash2021using}.}
\label{tab:valresults1}
\end{table}%
}

\blue{%
\begin{table}[!htb]
\centering
{\small
\begin{tabular}{|c|c|}
\hline
Problem & TC         \\ \hline
{JAK2} & 0.14 (0.032) \\ 
{DRD2} & 0.13 (0.034) \\  \hline
\end{tabular}}
\caption{Potential novelty of LLM-generated molecules using {\tt GPT-4o}. 
$TC$ represents the Tanimoto coefficient, ranging from 0 to 1, where 1 signifies high similarity to known inhibitors, while 0 indicates complete structural dissimilarity.
The entries represent the mean values, with standard deviations shown in parentheses.}
\label{tab:valresults2}
\end{table}%
}

\subsection{Case Study 2: Less Understood Target with Few Ligands}
\label{sec:case2}

We evaluate the performance of \GM in an open-ended setting where the
true target-site is not known precisely, and multiple factors have to be
optimised (binding affinity, molecular weight and synthesis accessibility),
and there are very few known inhibitors.
%
Note: Background Knowledge; Algorithms and Machines; and Methods are the same as
for Case Study 1 (Section~\ref{sec:case1}).

\subsubsection*{Problem}

 \noindent
{\bf DBH Inhibitors.} Human dopamine $\beta$-hydroxylase (DBH) is an enzyme that converts
dopamine (DA) to norepinephrine (NA) and plays a pivotal role in regulating the
concentration of NA, deficiency or overproduction of which causes several diseases related to the brain and the heart. This enzyme is thus of high therapeutic significance. The availability of the three-dimensional structure of DBH is expected to facilitate the identification of DBH
active-site inhibitors. In the meantime, the crystal structure of a dimer of DBH has been determined, providing insights into its function and aiding in the design of inhibitors \citep{vendelboe2016crystal}.
Specifically, we will use the \textit{in silico} model of the dimer to generate small molecules
with similar or better IC50 and KD values (in simulation) than at least one of the latest
generation of DBH inhibitors. We will use as data the 5 known DBH inhibitors: Tropolone, Disulfiram,
Nepicastat, Zamicastat, Etamicastat. The last three are shown in Fig.~\ref{fig:dbh_data}.
Tropolone is a naturally occurring
molecule, with known toxic effects. Disulfiram is a 1st generation molecule, and
also with toxic side-effects.  The last two molecules,
Zamicastat and Etamicastat are
the latest generation of DBH inhibitors and are currently in double-blind human
trials for hypertension. We focus on obtaining molecules with docking scores at
least as good as Nepicastat, a 4th generation drug.


\begin{figure}[!htb]
    \centering
    \includegraphics[width=0.9\textwidth]{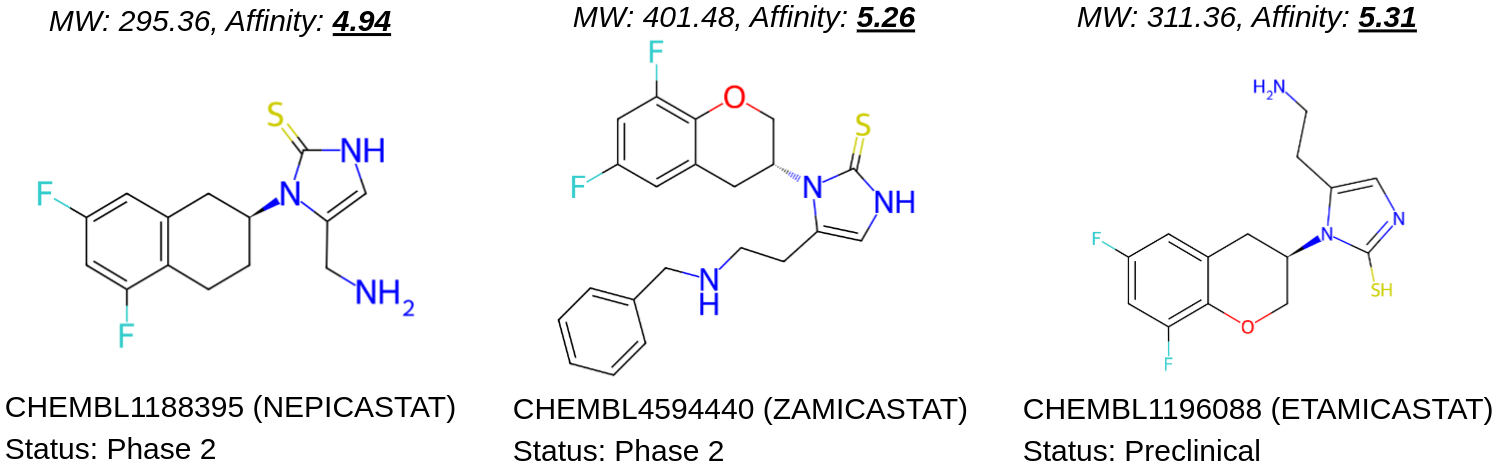}
    \caption{Three known DBH inhibitors at different stages of their FDA approval status. 
    Each compound is identified by their CHEMBL ID and name.
    `MW' refers to molecular weight;
    `Affinity' refers to the binding affinity predicted by
    {\tt GNINA} software while docking the molecules to the
    DBH protein, {\it 4zel}.
    The approval status of these compounds were noted 
    as on 25 January 2025 from the ChEMBL online portal.}
    \label{fig:dbh_data}
\end{figure}

 \subsubsection*{Results}
The exploratory problem concerns generating potential leads for
DBH inhibition, given data on the structure and inhibitory values of
5 molecules on a proxy target to DBH. The inhibitory efficacy of
a small number of molecules is known (see description of the problem above). We consider two
kinds of exploratory experiments. First, the LLM is provided with the
information about the known molecules and their structure: in LLM
parlance, we are doing ``few-shot learning'' (correctly, we are using
the LLM to draw from a distribution conditioned on the known molecules).
We will call this ``In-the-Box'' exploration. Secondly,
we do not give the LLM any information about known molecules (``zero-shot learning'',
or ``Out-of-the-Box'' exploration).  The top-5 molecules obtained for each
kind of experiment is shown in Fig.~\ref{fig:expresults1}.\\[2pt]

\begin{figure}[!htb]
    \centering
    \includegraphics[width=\textwidth]{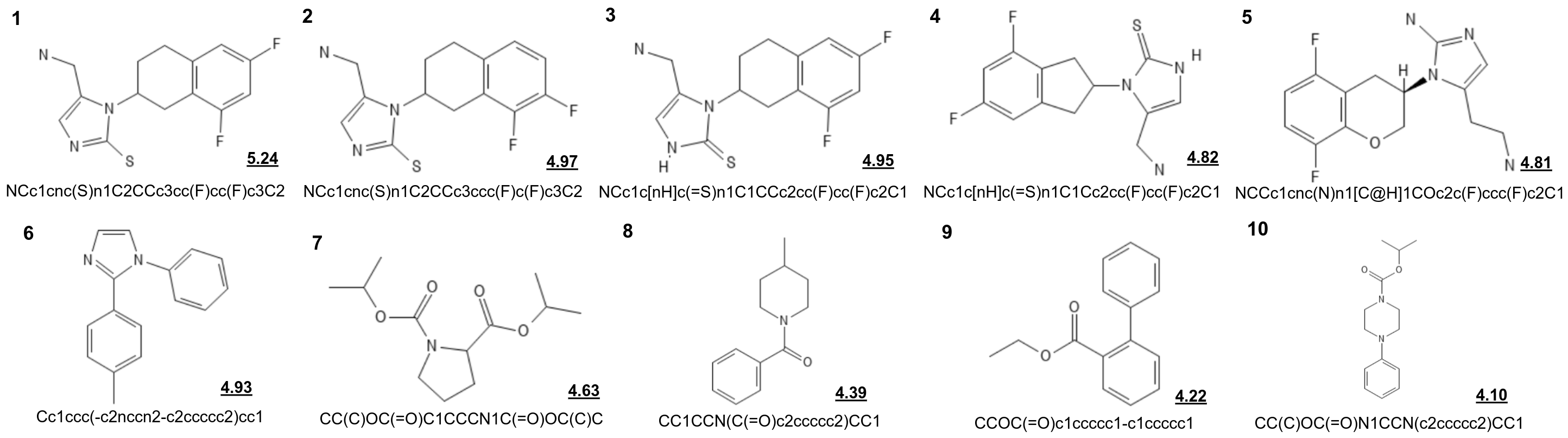}
    \caption{Potential inhibitors for DBH proposed by \GM, along with their
    binding score to \textit{4zel} protein.
    Molecules 1--5 are the top-5 molecules (ordered by estimated affinity)
    from  ``In-the-Box'' exploration. Molecules 6--10 are from ``Out-of-the-box''
    exploration. In the former the LLM used by \GM has few-shot examples when it
    starts. In the later, no such information is provided, and LLM uses its
    underlying distribution over molecules are used to generate the molecules.}
    \label{fig:expresults1}
\end{figure}

\noindent
But are the molecules any good? Here are some assessments by specialists\footnote{Dr.~Dey is an expert in structural biology and drug discovery, specialising in the molecular modeling and cryo-EM analysis of membrane proteins to target cardiovascular and infectious diseases.\\
Dr.~Banerjee is an expert of synthetic organic chemistry. His research interests include  medicinal chemistry, drug discovery, green chemistry, supramolecular chemistry, and the development of
molecular diagnostics tools.}:

\begin{itemize}
\item[] \textbf{``In-the-Box'' Exploration.} The views of the specialists about
        Molecules 1--5 are as follows:
        \begin{description}
            \item[Structural Biologist.]
            \underline{Molecules 1--4}. These are (likely to be) good                   
                inhibitors of DBH since these molecules are structurally similar to dopamine and nepicastat.  Another good feature of these  molecules is that  it carries Fluorenes as well.
        \underline{Molecule 5}. Could be a better inhibitor of DBH since it is
            structurally similar to dopamine and nepicastat. It may work. It could be better than the four other molecules since it carries one oxygen
            in the aromatic ring.  Another good feature of this molecule is that it also carries Fluorenes.
        \item[Synthetic Chemist.] All molecules can be synthesized, but is
            likely to be a long synthesis, and costs will be high.
    \end{description}

\noindent
\item[] \textbf{``Out-the-Box'' Exploration.} The views of the specialists about
        Molecules 6--10 are as follows:
        \begin{description}
            \item[Structural Biologist.] 
            \underline{Molecule 6}. This is a novel inhibitor of DBH where halogen is absent, 
            OH/O is also absent. Not similar to dopamine or nepicastat.
            Thus, it can be an allosteric inhibitor of DBH.
            \underline{Molecules 7,8}. This is also a novel inhibitor of DBH where
            halogen is absent. Since O is present twice, the affinity might be
            tighter. Aromatic rings are very dissimilar to dopamine or nepicastat.
            Thus, mode of inhibition is difficult to predict.
            \underline{Molecule 9}. May not be a good inhibitor due to highly flexible sidechains.
            \underline{Molecule 10}. This can be a better inhibitor than the
                first one since its aromatic rings are very interesting to bind in the active site of DBH. O= is present and thus the affinity might be tighter. Mode of inhibition should be competitive.
    \item[Synthetic Chemist.] All molecules can be synthesised, and synthesis
            is likely to be through a short route. The molecules may be commercially
            available.
    \end{description}
\end{itemize}

\noindent
The broad takeaways are these: (a) Unsurprisingly, In-the-Box exploration appears
to yield molecules that are similar (but not the same) as existing inhibitors.
In contrast, it is very interesting that Out-of-the-Box exploration appears to
yield very different molecules to known inhibitors;
(b) There are good biological reasons to expect Molecules 1--5 to bind to the
target, and Molecules 6,7,8 and 10 to bind to the target. Of these, the biologist
believes Molecules 6 and 10 to be especially interesting; and (c) On the synthesis
side, all molecules appear to be synthesisable, but 6--10 appear to be
more amenable to a short (and therefore, possibly cheaper) route.


\subsection{Additional Experimental Observations}
\label{sec:resultsextra}
We now consider some questions that are of relevance to both case studies,
and of wider interest to the uses of SNG:

\vspace*{0.2cm}

\noindent
\textbf{Is symbolic learning useful?}
The motivation for SNG was that for certain kinds of problems, it is
beneficial for symbolic learning to (a) ``focus'' the neural generator; and
(b) act as a touchstone for human assessments of the generator. It is relevant
to assess whether these aspects are reflected in the case-studies presented.
Table~\ref{tab:symutility} shows the results from just using the LLM-based generator
without a symbolic theory to constrain its output. The results suggest
the symbolic hypothesis does appear to play a useful role. It is unclear to
us why the variance with the use of symbolic theories is higher (we would have
expected it to be lower).

\begin{table}[!htb]
    \centering
    {\small
    \begin{tabular}{|c|cc|}
    \hline
    \multirow{3}{*}{Problem} & \multicolumn{2}{c|}{Mean Affinity}             \\ \cline{2-3} 
     &
      \multicolumn{1}{c|}{\multirow{2}{*}{\begin{tabular}[c]{@{}c@{}}Without Symbolic\\ Learning\end{tabular}}} &
      \multirow{2}{*}{\begin{tabular}[c]{@{}c@{}}With Symbolic\\ Learning\end{tabular}} \\
                             & \multicolumn{1}{c|}{}            &             \\ \hline
    JAK2                     & \multicolumn{1}{c|}{4.53 (0.43)} & 7.72 (1.53) \\ \hline
    DRD2                     & \multicolumn{1}{c|}{5.40 (0.60)} & 7.40 (0.61) \\ \hline
    DBH                      & \multicolumn{1}{c|}{3.80 (0.37)} & 4.72 (0.30) \\ \hline
    \end{tabular}
    }
    \caption{Estimated binding affinities of molecules generated without and with
        symbolic learning.}
    \label{tab:symutility}
\end{table}%

It is noteworthy that biologists and chemists examine molecules based on the
symbolic description -- in particular constraints on molecular weight and binding affinity.
They are able to comment meaningfully on instances generated by the machine. These two
findings, although qualitative, are important, as they indicate that both symbolic and
neural descriptions are important and intelligible to human experts.
Additionally,
not included here for reasons of space, is a further round of `molecule-exchange'
between the chemist and the biologist, inspired by Molecules 1--5. The
chemist proposed edited versions with shorter synthesis steps, and the
biologist commented further on the biological suitability of those molecules.

A key feature claimed for symbolic learning is the ability
to generalise from small amounts of data. In \citep{muggleton2023deeplog},
small program-synthesis tasks are used to demonstrate how useful symbolic learning
is possible with even a single positive example. SNG 
inherit the same ability, since the learning is confined to the
symbolic component, as is demonstrated in the 
case study of DBH inhibitors, which contains just 5 known positive
examples. Further experimental evidence of this ability is provided
in tabulation below.




\begin{table}[!htb]
\centering
{\small{
\begin{tabular}{|c|c|c|c|} \hline
$|E^+|$ & JAK2 & DRD2  & DBH \\ \hline
0     & 7.72 (1.53) & 7.40 (0.61) & 4.72 (0.30)      \\ \hline
1      &  8.83 (0.21)   & 6.15 (0.47) & 5.16 (0.13)   \\ \hline
2      & 8.69 (--)    & 6.99 (0.31) & 4.64 (0.48)   \\ \hline
3      &  8.50 (0.01) & 7.15 (0.58) &  5.42 (0.30)  \\ \hline
4      & 7.32 (0.77) & 5.60 (0.79) & 5.12 (0.18)   \\ \hline
5      & 8.59 (0.23) & 7.53 (0.47) & 4.23 (0.41) \\ \hline
 \end{tabular}
 }}
\caption{SNG with very small datasets. The results show estimated
mean inhibition for the targets, with very small numbers of positive
examples (and no negative examples).}
\label{tab:lowdata}
\end{table}

Qualitatively we observe that as
$|E^+|$ shrinks the symbolic constraints become broader and the generated
molecules less target-specific, though the hypothesis-search remains
informative even with a single positive example. 

\vspace*{0.2cm}
\noindent
\textbf{Does the choice of LLM matter?}
We considered SNGs with 2 leading general-purpose LLMs: OpenAI's GPT-4o
and Anthropic's Claude 3.5 Sonnet \citep{anthropic2024claude35sonnet}.
Details of the experiments are in Sec.~\ref{app:results} of the Appendix.
The results show that neither LLM seems clearly better. It is also relevant to examine
if such large, general-purpose LLMs are needed at all.
We also examined the performance of SNG with a specialised, small language
model (see the Appendix for details). Specifically,  we implemented an SNG variant
using 87M parameter GPT-2 model~\citep{entropy_gpt2_zinc_87m},
The results show that general-purpose large models are better on the benchmark
    datasets of JAK2 and DRD2. But interestingly, for the DBH problem -- where the target is
    less understood and known inhibitors are fewer -- the smaller specialised model
    appears to perform better. This suggests the larger models may be performing
    better simply by virtue of having been exposed to the benchmark data during training.\\[2pt]

\noindent
\textbf{Are the results sensitive to initial conditions/parameter settings?}
We again refer the reader to Sec.~\ref{app:results} of
the Appendix for detailed tabulations on the effect of initial conditions
and parameter changes. We are specifically interested in the effect
of the initial context provided to the LLM, and values for 2 parameters:
the \textit{sample size} used by the symbolic
leaner and the LLM's \textit{temperature}. On the benchmark data, the results appear
unaffected by changes in the initial context. However, for the exploratory dataset,
providing few-shot instances in the initial context does make a difference to the
similarity to known data instances. On parameter settings, we find GPT-4o output
to be the most sensitive to
parameter changes. SNG with Claude and our specialised GPT-2 model
were stable across changes in sample size and temperature.

\section{Concluding Remarks}
\label{sec:concl}

 To date, the state of AI and ML has been in what Kuhn referred to as 
 a \textit{pre-paradigm} state, with two main competing  schools of thought -- symbolic and connectionist -- offering different procedures, results and metaphysical principles.
 However, with startling results using connectionist techniques, we may well be on the verge of consensus developing on a Kuhnian notion
 of ``normal science'', at least in the machine-learning aspects of AI. Even so,
 there still appears to be a useful role for the tools and techniques developed by the symbolic school, especially in the use of ML for problems requiring the
 inclusion of prior knowledge, logical reasoning, correctness guarantees,
 and human-understandable explanations. For a class of problems with these characteristics, we propose a neurosymbolic approach called symbolic
 neural generation, or SNG. SNG systems are characterised by the use of a
 machine-constructed symbolic description that constrain the generation of data by 
 neural generators. A useful analogy is offered by the
 roles of specification and implementation in formal methods for software development.
 The symbolic model in SNG acts as the specification and the neural generator as
 the implementation. As is done in formal methods, we ensure that any output true
 of the implementation -- instances generated -- is also true of the specification (satisfies
 the constraints of the symbolic model). Unlike with
classic formal methods, in SNG ML techniques play a role in obtaining both
the specification and the implementation. We have proposed using the
mathematical structure of fibered-posets by way of specifying the
codomain of an SNG system. The paired elements comprising this set
makes it a natural choice for hybrid neurosymbolic systems, where one
element of the pair is of symbolic and the other of neural origin.

 In the paper, a practical motivation for SNG is provided by the need to accelerate the
 identification of `leads' in early stage drug-design.  Leads are small
molecules capable of binding to a target protein, and satisfying some
physico-chemical constraints. The specific area where ML could help is
in cases where the structure of target site is not well understood, and
there are as yet very few small molecules that have been shown to be good
inhibitors. From a ML standpoint, this constitutes an ideal situation for
symbolic learning. We also want to be able to generate entire molecules that
can be examined critically by synthetic chemists and structural biologists. This
generation requires a complex probability model that can adequately represent
the molecular patterns of interest, and also efficient mechanisms for sampling from
the distribution. 
Neural generators -- especially those based on language models -- appear
to be especially well-suited for this. 
This is our motivation for developing SNG. We have provided results on benchmark problems that
show SNG to be comparable to state-of-the-art methods. 
However, it is the exploratory
study that we think is substantially more relevant to the area of early-stage drug-design. Specifically, it shows: (a) the use of very small numbers of data instances
(5, in this case); (b) some of the
results from the generator -- especially in the ``Out of the Box'' mode -- may
be biologically novel and can be synthesised.  While
the output has been made intelligible to the specialist through the
use of an LLM, controlling and verifying the LLM's output has been made possible
through the use of the symbolic model. 

There are several ways in which the work here can be improved and extended.
Immediately of relevance  are improvements to the
\GM algorithm. We have already listed a number of limitations of the procedure:
addressing each one will improve both the quality of solutions returned,
and the efficiency of the approach. Additional studies of the
``out-of-the-box'' kind will also help establish SNG as a useful tool for
early-stage drug design.
Further exploration is also needed on the symbolic side.
The use of probabilistic symbolic learning for example, 
may be needed for problems for which background knowledge is uncertain or
missing. Nothing in the conceptual framework of SNG systems requires
the use of categoric logical descriptions.
More broadly, SNG systems are of relevance not just to molecule-generation,
but for any problem requiring the generation of data that needs verification
either formally or by a person, but for which we do not already have
\blue{a formal} description (generation of system's behaviour, subject to the
constraints of symbolically-learned digital twins; planners where
a symbolic model is used to constrain and verify plans generated by an LLM;
and symbolic models constraining the generation of experiments for
active learning are some examples).  Finally,   we believe the conceptual
structure of fiber posets are naturally applicable to 
the development of other kinds of hybrid neurosymbolic systems. For example,
when developing neural-predictors with symbolic-explainers, it is
evident that we are dealing with pairs of models. Whether there exist
fiber poset constructions for such hybrid models needs further exploration.

\subsection*{Code and Data Availability}
The data and code are available at: \url{https://github.com/tirtharajdash/LMLFStar}.

\subsection*{Acknowledgements}
During part of this work, AS was a visiting Professorial Fellow at UNSW, 
and a TCS Affiliate Professor. He is a member of
the Anuradha and Prashant Palakurthi Centre for AI Research (APPCAIR) at BITS Pilani.
This research is partly supported by: DBT project BT/PR40236/BTIS/137/51/2022 ``Developing Predictive Models for `druglikeness' of small molecules''; 
and CDRF project  C1/23/184 ``Silicon-to-Lead: AI-Driven Design, Synthesis and Development of New Drugs to Combat Cardiovascular Diseases''.
The authors would like to acknowledge Aaron Rock Menezes
for his implementation of the PyLMLF algorithm reported in
\citep{brahmavar2024generating}.
The authors sincerely thank Professors Suman Kundu and Sumit Biswas for their insightful discussions on DBH.
The authors used Claude (Opus 4.7, Anthropic) for language editing and consistency-checking during preparation of the revised manuscript.

\begin{appendix}
\label{app:proofs}
\section{Additional conceptual details for SNG}
\subsection{Note on extending other hybrid systems}
\label{app:hybridext}

The general principle of a symbolic
base combined with neural-fibres can be applied to characterise
systems in each of the cases (A)--(D) in Fig.~\ref{fig:categ_table}.

\begin{myexample}[Hybrid systems with a symbolic base]
Let us assume the following:
\begin{itemize}
    \item A fixed universal set $\mathcal U$ (instances, embeddings, predictions, programs \textit{etc\/}).
    \item Background knowledge $B$;
    \item A symbolic component that can enumerate elements
    from a set of symbolic hypotheses $\mathcal H$; and
    \item A neural component that -- given any element from $H \in {\cal H}$ -- can
    generate elements from ${\cal U}$ that are consistent with $H$.
\end{itemize}
Additionally, the  set ${\cal H}$ is assumed to a partially ordered
set, with ordering $\ge_{\cal H}$. Associated with each element $H$ in the
ordering is a fibre-poset $F(H)$ that is obtained from the neural component.
Each element of the base set is combined with elements from the fibre-sets to
give a partially ordered set ${\mathcal F}$ of  hybrid systems.
Example hybrid systems in the categorisation in Fig.~\ref{fig:categ}(a)
that can be characterised in this way are:

\vspace*{0.2cm}
\centering
{\scriptsize{
\begin{tabular}{|c|l|l|p{2.6cm}|p{3.5cm}|p{3,5cm}|}
\hline
\textbf{Case} &
\textbf{Symbolic} &
\textbf{Neural} &
\textbf{Base} &
\textbf{Fibre} &
\textbf{$\mathcal F$} \\
\hline
\textbf{A} &
Reasoning &
Reasoning &
Hypotheses ordered by entailment &
Neural reasoners consistent with the symbolic entailments &
Pairs ordered by stronger symbolic entailment and higher neural agreement \\
\hline
\textbf{B} &
Learning &
Reasoning &
Hypotheses ordered by specialization &
Neural approximators of the symbolic extensions &
Pairs ordered by more specific hypotheses and more confident neural outputs \\
\hline
\textbf{C} &
Reasoning &
Learning &
Hypotheses ordered by entailment &
Neural representations trained to preserve symbolic reasoning patterns &
Pairs ordered by stronger entailment and representational inclusion \\
\hline
\textbf{D} &
Learning &
Learning &
Hypotheses ordered by generality &
Neural learners minimizing a loss defined by the symbolic hypothesis &
Pairs ordered by more general hypotheses and lower training loss \\
\hline
\end{tabular}
}}

\end{myexample}

This example shows how the setting
of a symbolic base with neural fibres can be used to
characterise other hybrid systems where the symbolic component
has the primary role. Further variants -- of less  relevance
to this paper -- but of wider applicability refer to: (a) Hybrid
neurosymbolic systems where the neural component has the primary
role \blue{(a base of neural states with symbolic fibres)}
;
and (b) Hybrid systems where both neural and symbolic components are equally
important. \blue{In this case the base consists of pairs and each pair
is associated with a hybrid instance-set.}
\footnote{
This applies when the system is co-trained, or outputs are coupled.
The ordering over ${\mathcal F}$ can then be very flexible,
not just the obvious  product orders, but also
problem-specific orders.
}
\subsection{Correctness of Procedure \Gen}
\label{app:correctnessGen}
\setcounter{myprop}{1}
\begin{myprop}
The set $M_n$ returned by Procedure \ref{alg:nstar}  is an
element of $F(H)$ and $w_n \in [0,1]$. 
\end{myprop}
\begin{proof}
\Gen executes the loop (Steps \ref{step:startloop}--\ref{step:endloop})  $n$ times.
We claim the following is loop invariant:
\[M_{i-1} \subseteq N(H) \cap \operatorname{ext}(H|B) \mbox{ and } |M_{i-1}| \leq s \times (i-1)\] 
At the start of the iteration ($i = 1$), $M_{i-1} = \emptyset$
and the invariant is trivially true.
Assume the invariant holds at the start of the $k^\text{th}$ iteration. That is,
$M_{k-1} \subseteq N(H) \cap \operatorname{ext}(H|B)$ and 
$|M_{k-1}| \leq s\times(k-1)$.
On the $k^\text{th}$ iteration: 
    (i) $S_k \subseteq N(H)$;
    and (ii) in Step \ref{step:update} $M_k = Z_k \cup M_{k-1}$,
    where $Z_k = \{x:  x \in {\cal U},~ (true,x) \in D_k\}$. Since
    $(true,x) \in D_k ~\text{iff}~ x \in S_k \cap \operatorname{ext}(H|B)$. So $Z_k
    \subseteq N(H)~ \cap~ \operatorname{ext}(H|B)$.
    Therefore $M_k \subseteq N(H)~ \cap~ \operatorname{ext}(H|B)$.
    Also since $|D_k| \leq s$,
    $|M_k| \leq s + |M_{k-1}|$. Since 
    $|M_{k-1}| \leq s \times (k-1)$, it follows that
    $|M_k| \leq s \times k$. 
    The loop variable $i$ is incremented to $k+1$, and at the start of the next
    iteration, clearly $M_{i-1}=M_k$ which is a subset of $\operatorname{ext}(H|B)$ and has at most $s \times k$ elements. The procedure
    clearly terminates since $i$ is bounded by $n$, and
    the procedure returns the set $M_{n}$ which is a subset of
    $\operatorname{ext}(H|B)$ and $w_n \in [0,1]$. 
\end{proof}

\section{Algorithmic Details: Application to Lead Discovery}
\label{app:alg}
\subsection{Setting up the base poset}

We first introduce some definitions that are helpful in
clarifying both  the set of hypotheses ${\cal H}$ and
the ordering over that set.

\begin{mydefinition}[Interval-Vectors]
An $n$-dimensional interval-vector $\mathbf{v}$ $=$
$([a_1,b_1],\ldots,[a_n,b_n])$ is an
element of
$(\mathbb{R} \times \mathbb{R})^n$ where  $a_i \leq b_i$
for $i \in \{1,\ldots,n\}$.
\end{mydefinition}

\noindent
 We will sometimes denote  $(\mathbb{R} \times \mathbb{R})$ as ${\cal I}$
and $(\mathbb{R} \times \mathbb{R})^n$ as ${\cal I}^n$.
The set ${\cal I}^n$ is therefore the set of $n$-dimensional hyper-rectangles,
and an interval-vector is a hyper-rectangle.

\noindent
The following definition is useful later.

\begin{mydefinition}[Interval-Vector Containment]
\label{def:subsume}
 Given interval-vectors $\mathbf{v}_1, \mathbf{v}_2 \in {\cal I}^n$
If
$(\mathbf{v}_2[1] \subseteq \mathbf{v}_1[1]) \wedge \cdots \wedge (\mathbf{v}_2[n] \subseteq \mathbf{v}_1[n])$ then
we will say $\mathbf{v}_2$ is contained by $\mathbf{v}_1$ (resply. $\mathbf{v}_1$ 
contains $\mathbf{v}_2$)
We denote this by
$\mathbf{v}_2 \sqsubseteq \mathbf{v}_1$ (resply.
$\mathbf{v}_1 \sqsupseteq \mathbf{v}_2$). If there exists at least
one $j \in \{1,\ldots,n\}$ s.t. $\mathbf{v}_2[j] \subset \mathbf{v}_1[j]$, we
will say $\mathbf{v}_2$ is properly contained by $\mathbf{v}_1$ (resply.
$\mathbf{v}_1$ properly contains $\mathbf{v}_2$). We denote
this by $\mathbf{v}_2 \sqsubset \mathbf{v}_1$ (resply. $\mathbf{v}_1 \sqsupset \mathbf{v}_2$).
\end{mydefinition}

\noindent
Clearly, if $\mathbf{v}_1 \sqsubset \mathbf{v}_2$ then
$\mathbf{v}_1 \sqsubseteq \mathbf{v}_2$.

\begin{mydefinition}[Factors]
\label{def:fac}
Let ${\cal U}$ be a set of instances. A factor is
a function $f: {\cal U} \to \mathbb{R}$.
\end{mydefinition}

\begin{mydefinition}[Factor Specification]
\label{def:facspec}
Let $F = (f_1,\ldots,f_n)$ be a sequence of factors.
A factor specification is the pair
$(F,\mathbf{\Theta})$ and
$\mathbf{\Theta} = ([f_1^-,f_1^+],\ldots,[f_n^-,f_n^+]) \in {\cal I}^n$.
For $i \in \{1,\ldots,n\}$, $[f_i^-,f_i^+]$ is the range of values for the
factor $f_i$.
\end{mydefinition}

\noindent

A factor-specification allows us to define the notion of
an \textit{experiment}:

\begin{mydefinition}[Experiment]
\label{def:expt}
Let ${\cal U}$ be a set of instances.
Let $((f_1,\ldots,f_n),\mathbf{\Theta})$ be a factor-specification.
An experiment $\mathbf{e}$ given the factor-specification, or simply an experiment,
is an interval-vector in ${\cal I}^n$ s.t.
$\mathbf{\Theta}$ subsumes $\mathbf{e}$.
\end{mydefinition}

Experiments specify conjunctive logical constraints on factors. We adopt
the terminology in Inductive Logic Programming (ILP) and each experiment will
 be associated with a \textit{hypothesis}.

\begin{mydefinition}[Hypothesis]
\label{def:genhyp}
Let ${\cal U}$ be a set of instances,
$(F,\mathbf{\Theta})$ be the factor specification where
$F = (f_1,\ldots,f_n)$. Let 
$\mathbf{e}$ be an experiment given $(F,\mathbf{\Theta})$.
Then a hypothesis given an experiment is the clause:

\begin{center}
{\footnotesize{
\begin{tabbing}
\= $P(x) \leftarrow$ \= \kill
\> $Hypothesis((f_1,\ldots,f_n),(i_1,\ldots,i_n)):$\\
\> $\forall x (\Sigma(x) \leftarrow$ \\
\> \> $x \in {\cal U} \wedge $\\
\> \> $(f_1(x) \in i_1) \wedge \cdots \wedge (f_n(x) \in i_n))$
\end{tabbing}
}}
\end{center} 
\end{mydefinition}

\noindent
The hypothesis space ${\cal H}$ is the set of such hypotheses.
The ordering on this set will be same as introduced in Defn.~\ref{def:pohyp},
namely pointwise inclusion of extensions.

\subsection{Bayesian Scoring of Hypotheses}
\label{app:Qscore}

\begin{mydefinition}[McCreath's $Q$-Heuristic]
\label{def:q}
Let ${\cal U}$ denote the set of all instances.
Let $h$ be a hypothesis as defined in Defn.~\ref{def:genhyp}.
Let $E^+$ denote a set of positive  examples
and $E^-$ denote a set of
negative examples s.t. $(|E^+| \geq 1$ and $|E^-| \geq 0$.
Let $D = E^+ \cup E^-$. 
Let $\text{ext}(h|B) = \{x: x \in {\cal X}, B \wedge h \models Feasible(x)\}$; and for
any $S \subseteq {\cal U}$, 
$\theta(S) = \frac{|S|}{|{\cal X}|}$.
Let $\epsilon$ be
the probability that an instance is randomly assigned to $E^+$ (resply.
$E^-$).
Let $B$ denote background knowledge, $TP(H|B,D) = \{e: e \in E^+, e \in \operatorname{ext}(H|B)\}$;
$TN(H|B,D) = \{e: \neg e \in E^-, B \wedge H \wedge e \not \in \operatorname{ext}(h|B)\}$; and
$FPN(H|B,D) =  D \setminus (TP(H|B,D) \cup TN(H|B,D))$.
 Then, dropping the inclusion of $B,D$ for convenience,
 the fixed-example model in \citep{mccreath1998lime}
 defines the quality of a hypothesis as:
\begin{align*}
Q(H) = & \, \operatorname{log}(P(H)) \\
& + |TP(H)| \, \operatorname{log}\left( \frac{1-\epsilon}{\theta(\operatorname{ext}(H))} + \epsilon \right) \\
& + |TN(H)| \, \mathrm{log}\left( \frac{1-\epsilon}{1-\theta(\operatorname{ext}(h))} + \epsilon \right) \\
& + |FPN(H)| \, \operatorname{log}(\epsilon).
\end{align*}

\noindent
For the special case of $\epsilon = 0$, the quality
of a hypothesis in the fixed-example setting simplifies to:
\[
Q(H) = \operatorname{log}(P(H)) + |E^+| \, \operatorname{log} \frac{1}{\theta(\operatorname{ext}(h))} + |E^-| \, \mathrm{log} \frac{1}{1-\theta(\operatorname{ext}(h))}.
\]
\end{mydefinition}

In \citep{mccreath1998lime} it is shown that maximising
$Q(H|B,D)$ maximises the Bayesian
posterior $\text{P}(H|B,D)$, along other theoretical results including a
proof of (probabilistic) convergence to a target concept. 
Assuming the entailment relation $\models$ can be checked,
the practical difficulties in using the $Q$-heuristic are in obtaining the
values for $\theta(\operatorname{ext}(H))$ and $P(H)$. We note the following:

We will need the following to be able to use the $Q$-heuristic here:
\begin{enumerate}
    \item[a.] In order to obtain the sets $TP,TN$ and $FPN$ 
        we will require $B$ to contain all the definitions needed to evaluate
        the constraint in the hypothesis (that is, $B$ will
        need to contain definitions for the $f_i(\cdot)$).
    \item[b.] By definition, $\text{ext}(h)$ is the set of feasible instances
        as defined
        in $Defn. \ref{def:ext}$.
        We can estimate $\theta(ext(h))$ on a random sample $X \subset {\cal X}$
        as follows: Let $S = \{x: x \in X, \Phi(x) \text{ is } true\}$.
        Then the (maximum-likelihood) estimate of $\theta(\text{ext}(h))$ is
        $\hat{\theta}(ext(h)) = \frac{|S|}{|X|}$.\footnote{A sample-based estimate is also
        used in \citep{mccreath1998lime} and \citep{muggleton1996positive}.}
\end{enumerate}

\begin{myremark}
\label{rem:posonly}
These results follow straightforwardly
from the definition in Defn. \ref{def:q}:
    
\begin{description}
    \item[Positive-only data.] Let $E^+ \neq \emptyset$ and 
        $E^- = \emptyset$. Let $\epsilon = 0$.
        If $P(H_1) = P(H_2)$, then
        $(Q(H_1) > Q(H_2))$ iff $(\operatorname{ext}(h_1) < \operatorname{ext}(h_2))$. 
     \item[Negative-only data.] Let $E^- \neq \emptyset$ and 
        $E^+ = \emptyset$. Let $\epsilon = 0$.
        If $P(H_1) = P(H_2)$, then
        $(Q(H_1) > Q(H_2))$ iff $(\operatorname{ext}(H_1) > \operatorname{ext}(H_2))$.
\end{description}
That is, in the noise-free case, with equal prior probabilities,
and positive data only, more specific
hypotheses will be preferred; and with negative data only, more general
hypotheses will be preferred.
\end{myremark}

\subsection{Property of the procedure \GM} 
\label{app:proofcmimpl}
\begin{myprop}
Let $(F,\cdot)$ be a factor-specification, $B$ denote background knowledge.
    Let $\mathbf{e}_{i}$ ($1 \leq i \leq k$) be an experiment selected by
     \GM on the $i^\text{th}$ iteration s.t. $\mathbf{e}_i$ is contained by $\mathbf{e}_{i-1}$.
Let $H_i = Hypothesis(F,\mathbf{e}_i)$, and $H_{i-1} = Hypothesis(F,\mathbf{e}_{i-1})$.
Then $H_{i-1} \models H_i$.
\end{myprop}
\begin{proof}
First we observe that $H_{i}(x):  (\Sigma(x) \leftarrow (x \in {\cal U} \wedge C_i(x))$ and
$H_{i-1}(x):  (\Sigma(x) \leftarrow (x \in {\cal U}) \wedge C_{i-1}(x))$.
It is easy to see that  since by construction  $\mathbf{e_{i-1}}$ contains $\mathbf{e_i}$,
$\forall x (C_{i}(x) \models C_{i-1}(x))$. Now suppose $H_{i-1} \not \models H_i$.
That is, there exists some $a \in {\cal U}$ s.t. $H_{i-1}(a)$ is true
and $H_i(a)$ is false. Since $H_{i}(a)$ is false,
$\Sigma(a)$ is false and $C_{i}(a)$  is true.  Since $C_{i}(x) \models C_{i-1}(x)$ for all
$x$, and $C_i(a)$ is true, therefore $C_{i-1}(a)$ is true. Since $H_{i-1}$ is assumed
true, and $C_{i-1}(a)$ is true, then $\Sigma(a)$ is true, which is a contradiction.
So for all  $a \in {\cal U}$, whenever $H_{i-1}$ is true
then $H_i$ is also true. 
\end{proof}

\section{Additional Experimental Details: Case Studies}
\label{app:expt}

\subsection{Note on biochemical terminology}
\label{app:biochem}
\blue{
A few terms used throughout the paper may be unfamiliar to readers from
the ML community. 
A \emph{small molecule} is a low-molecular-weight
organic compound (typically below 900\,Da), often a drug candidate.
A \emph{target} is a biomolecule (usually a protein) whose activity one
wishes to modulate to produce a therapeutic effect. 
An \emph{inhibitor} is a small molecule that binds the target and reduces its activity,
typically by occupying a functional site such as the active site or an
allosteric pocket. 
\emph{Binding affinity} measures the strength of the
molecule--target interaction (commonly reported as $K_i$, $K_d$, or
IC$_{50}$); 
in this work we estimate it computationally using docking
software (\texttt{GNINA}), 
which predicts a docking score (binding affinity).
A \emph{hit} is a compound showing measurable activity against the
target in an initial assay, while a \emph{lead} is a hit that has been
validated and refined to exhibit suitable potency, selectivity, and
drug-like properties, serving as a starting point for further
medicinal-chemistry optimisation \citep{hughes2011principles,bleicher2003hit}.
A \emph{scaffold} is a common structural core shared across a family of
molecules, and \emph{SMILES} is a standard
string encoding of molecular structure.
Readers interested in a more detailed introduction to small-molecule drug design should see \citep{di2015drug,hughes2011principles,doytchnova:drug}.
Figure~\ref{fig:binding} illustrates the geometric situation that a
docking software attempts to score: a small-molecule inhibitor lodged
inside the binding pocket of its protein target.
}

\begin{figure}[!htb]
\centering
\begin{tikzpicture}[
    every node/.style={font=\normalsize},
    atom/.style={circle, fill=black, inner sep=1.1pt},
    bond/.style={thick},
]
  \draw[thick, fill=gray!25]
    (-3.2,-1.8)
    .. controls (-4.3,-0.6) and (-4.3,0.6) .. (-3.2,1.8)
    .. controls (-2.0,2.5) and ( 0.0,2.5) .. ( 1.2,1.8)
    .. controls ( 1.7,1.3) and ( 1.0,1.0) .. ( 0.5,0.6)
    .. controls ( 0.1,0.2) and ( 0.1,-0.2) .. ( 0.5,-0.6)
    .. controls ( 1.0,-1.0) and ( 1.7,-1.3) .. ( 1.2,-1.8)
    .. controls ( 0.0,-2.5) and (-2.0,-2.5) .. (-3.2,-1.8);

  \node at (-2.0,0.2) {\textbf{Protein target}};
  \node[font=\footnotesize, align=center] at (-2.0,-0.4)
        {(e.g.\ a kinase domain)};

  \begin{scope}[shift={(1.05,0)}, scale=0.35]
    \coordinate (a1) at ( 90:1);
    \coordinate (a2) at ( 30:1);
    \coordinate (a3) at (-30:1);
    \coordinate (a4) at (-90:1);
    \coordinate (a5) at (-150:1);
    \coordinate (a6) at ( 150:1);
    \coordinate (a7) at ($(a2)+(0.8,0.4)$);
    \coordinate (a8) at ($(a7)+(0.9,-0.2)$);
    \draw[bond] (a1)--(a2)--(a3)--(a4)--(a5)--(a6)--cycle;
    \draw[bond] ($(a1)+(-0.05,0.05)$)--($(a3)+(-0.05,-0.05)$);
    \draw[bond] ($(a3)+(-0.05,0.05)$)--($(a5)+(0.05,0.05)$);
    \draw[bond] (a2)--(a7);
    \draw[bond] (a7)--(a8);
    \foreach \p in {a1,a2,a3,a4,a5,a6,a7,a8} \node[atom] at (\p) {};
  \end{scope}

  \draw[dashed, thick] (0.50, 0.55) -- (0.95, 0.30);
  \draw[dashed, thick] (0.50,-0.55) -- (0.95,-0.30);
  \draw[dashed, thick] (0.10, 0.00) -- (0.70, 0.05);

  \draw[->, thick] (3.6, 1.6) -- (1.7, 0.8);
  \node[anchor=west, align=left] at (3.6, 1.6)
        {Binding\\ pocket};

  \draw[->, thick] (3.6, 0.0) -- (1.7, 0.0);
  \node[anchor=west, align=left] at (3.6, 0.0)
        {Small-molecule\\ inhibitor (ligand)};

  \draw[->, thick] (3.6,-1.6) -- (0.95,-0.15);
  \node[anchor=west, align=left] at (3.6,-1.6)
        {Non-covalent\\ interactions};
\end{tikzpicture}
\caption{Schematic of small-molecule binding. A drug-like inhibitor
(ligand) occupies a concave pocket on the surface of a protein target
and is held in place by non-covalent interactions (dashed lines).
A docking software predicts both the bound pose and an associated
binding-affinity score.}
\label{fig:binding}
\end{figure}

\subsection{Method}
\label{app:meth}

\noindent
The following additional details are relevant to the experimental method:

\begin{itemize}
    \item In the controlled (Validate) experiments, we are only attempting to optimise
        one factor, namely: docking score, which is indicative of binding affinity.
        This is in line with what was done in \citep{brahmavar2024generating}.
        For the open-ended (Explore) experiments, we will extend this to include:
        number of synthesis steps, and estimated yield per step.
    \item The factor-set specification also requires identifying minimum and
        maximum for the initial search space. 
        For all experiments, we use: $\{\mathit{affinity}: [3,10], \mathit{molwt}: [200, 700], \mathit{SAS}: [0, 7.0]\}$,
        where $\mathit{affinity}$ is the predicted affinity from GNINA software,
        $\mathit{molwt}$ is the molecular weight, and 
        $\mathit{SAS}$ denotes synthesis accessibility score.
    \item For the controlled experiments, we use the known inhibitors and non-inhibitors as the dataset $D$. For the open-ended experiments we use 5 known inhibitors of DBH
    and 5 randomly sampled molecules from ChEMBL as non-inhibitors. 
    Estimating the $Q$-heuristic requires a sample of
        unlabelled molecules. For this we use a randomly drawn set of 1000
        molecules from the ChEMBL database.
    \item The description of \SH does not specify a sampling method for
            obtaining subsumed hyper-rectangles. We use an approach based
            on Latin Squares Hyper-Rectangle Sampling (LHRS: \citep{mckay2000comparison}).
            Some additional prior information may be available that may be used
            to modify the basic LHRS approach: (a) If we know beforehand that
            a factor is to maximised (for example, binding affinity), then we 
            do not sample points from the upper-end of the range for the factor.
            Thus, subsuming rectangles are obtaining by only random placements of
            the lower-end; (b) Similarly, if we know beforehand that we want to minimise
            a factor, then only the upper-threshold is sampled. If nothing is known
            then a standard LHRS approach is adopted.
    \item For all experiments, we use a value of $s=10$; and $n=10$ for \GM and we sample
    100 molecules when it returns the optimal hypothesis. The set of feasible molecules
    during search and in the final generation are considered for evaluation.
    \item For molecule generation from LLM, we use model's chat-completion API with a maximum output length of $128\times \mathtt{max\_samples}$ tokens, a temperature of 0.7 to balance diversity and determinism. The prompt content was provided via
    the $\mathtt{messages}$ parameter (see Appendix~\ref{app:expt} for details).
    \item We assess the results of controlled experiments by examining the range and
        median docking scores of the molecules generated, and compare those to
        those generated by the LMLF procedure in \citep{brahmavar2024generating}. For the open-ended
        experiments, we obtain the statistics on docking scores, and compare them
        to those of the latest generation of molecules used for DBH inhibition
        (see Sec. \ref{sec:case2}). In addition, we also provide an assessment of the
        molecules by an expert synthetic chemist.
    \item For each target problem, we assess the novelty of the generated molecules
        by using the average Tanimoto (or Jaccard) coefficient to the database of  
        known inhibitors. 
\end{itemize}

\subsection{Results}
\label{app:results}

\noindent
The following details refer to the questions in Sec.~\ref{sec:resultsextra}.

\vspace*{0.2cm}
\textbf{The choice of LLM.}
Table~\ref{fig:llms} compares the performance of \GM when
implemented with two different general-purpose LLMs: OpenAI's GPT-4o
and Anthropic's Claude 3.5 Sonnet \citep{anthropic2024claude35sonnet}.
Both these \GM variants are evaluated under identical parameter 
settings and prompts to ensure a fair comparison.
Results are shown for both the zero early context ($|C_0|=0$) and early context ($|C_0|=5$)
settings across the three benchmark protein 
targets. Overall, the two models yield comparable mean binding affinities, 
with neither showing a clear advantage from the inclusion of early context.
This similarity in performance between
the two general-purpose LLMs raises the question whether a 
domain-specific LLM might offer a potential advantage.
We examine this next.

    We implement a \GM variant using 87M parameter variant of GPT-2 that has been trained
    on a large corpus of about 480 million molecules (as SMILES strings) from the ZINC database \citep{entropy_gpt2_zinc_87m},
    which provides extensive coverage of chemical space and encodes 
    general chemical grammar and structure-property relationships.
    We refer this model here as Molecule-GPT2.
    This model is expected to generate syntactically 
    valid and chemically diverse molecules 
    due to its broad exposure to molecular representations.
    However, because the training data does not contain 
    information specific to the protein targets considered
    here, the model may be lacking the ability to exploit
    target-specific binding patterns that are otherwise
    known to the general-purpose LLMs through their 
    enormous training corpus of research articles and large-scale studies. Nevertheless, Molecule-GPT2 produces
    competitive predicted affinities for JAK2 and DRD2, which are probably very highly studied proteins and for these substantial
    domain-knowledge is available from the literature.
    Interestingly, it outperforms general-purpose LLMs for DBH, a relatively lesser studied target and additional target-specific
    knowledge has to be inferred from known inhibitor molecules.
    
    Although these results indicate that a general-purpose LLM is 
    more suited for target-specific lead generation in the manner
    done in this study, a more thorough investigation is warranted.
    Furthermore, while prompting was used directly to generate
    molecules from the general-purpose LLMs, the same approach
    was not straightforward for the domain-specific LLM (Molecule-GPT2),
    primarily because it was not trained in the same manner as GPT-4o or Claude 3.5 Sonnet.
    For each protein target, we first had to construct a small
    set of structural scaffolds by examining known inhibitor
    molecules, and then use these scaffolds as prefixes for 
    GPT2-based molecule generation (see Appendix~\ref{app:expt} for 
    more details on this).
    This additional preprocessing step may have contributed
    to the lower performance observed for Molecule-GPT2 
    relative to the general-purpose LLMs.
    
    \begin{table}[!htb]
    \centering
    {\small
    \begin{tabular}{|c|c|c|c|}
    \hline
    Problem               & $|C_0|$ & GPT-4o    & Claude 3.5 Sonnet  \\ \hline
    \multirow{2}{*}{JAK2} & 0       & 7.72 (1.53) & 8.41 (0.25) \\ \cline{2-4} 
                          & 5       & 8.59 (0.23) & 8.43 (0.22) \\ \hline
    \multirow{2}{*}{DRD2} & 0       & 7.40 (0.61) & 7.44 (0.48) \\ \cline{2-4} 
                          & 5       & 7.53 (0.47) & 7.43 (0.42) \\ \hline
    \multirow{2}{*}{DBH}  & 0       & 4.72 (0.30) & 4.67 (0.63) \\ \cline{2-4} 
                          & 5       & 4.23 (0.41) & 4.53 (0.58) \\ \hline
    \end{tabular}}
    \caption{Comparison of mean (standard deviation) predicted binding affinities
    for molecules generated by SNG using two LLMs: OpenAI's GPT-4o 
    and Anthropic's Claude 3.5 Sonnet.
    Both models are evaluated without early context ($|C_0|=0$) and 
    with early context ($|C_0|=5$) settings across three protein targets (JAK2, DRD2, DBH).
    For both LLMs, the temperature parameter was fixed at $0.7$.
    }
    \label{fig:llms}
    \end{table}

    \begin{table}[!htb]
    \centering
    {\small
    \begin{tabular}{|c|c|c|c|}
    \hline
    Problem & GPT-4o      & Claude 3.5 Sonnet & Molecule-GPT2 \\ \hline
    JAK2    & 7.72 (1.53) & 8.41 (0.25)       & 6.50 (0.60) \\ \hline
    DRD2    & 7.40 (0.61) & 7.44 (0.48)       & 6.99 (0.61) \\ \hline
    DBH     & 4.72 (0.30) & 4.67 (0.63)       & 5.74 (0.41) \\ \hline 
    \end{tabular}}
    \caption{Comparison of mean (standard deviation) predicted binding affinities
    for molecules generated by  SNG variants (with zero early context) 
    and an 87M parameter GPT2 model trained
    with 480M molecules from the ZINC database.}
    \label{fig:llmvars}
    \end{table}

\vspace*{0.2cm}
\noindent
\textbf{The effect of parameters.}
    We consider 3 parameters: (a) Initial context provided to the LLM;
    (b) sample-sizes used by the symbolic learner; and (c) LLM
    temperature.
    Table~\ref{appfig:context} shows the changes in mean binding affinity
    with the initial context (``few-shots'') provided to the LLM.
    The results suggest that there is little difference in
    predicted affinity, although there is a large change in the Tanimoto
    coefficient for the DBH data. 

\begin{table}[!htb]
\centering
{\small
\begin{tabular}{|c|c|c|c|}
\hline
Problem               & $|C_0|$ & Affinity    & $TC$        \\ \hline
\multirow{2}{*}{JAK2} & 0       & 7.72 (1.53) & 0.14 (0.032) \\ \cline{2-4} 
                      & 5       & 8.59 (0.23) & 0.15 (0.014) \\ \hline
\multirow{2}{*}{DRD2} & 0       & 7.40 (0.61) & 0.13 (0.034) \\ \cline{2-4} 
                      & 5       & 7.53 (0.47) & 0.13 (0.033) \\ \hline
\multirow{2}{*}{DBH}  & 0       & 4.72 (0.30) & 0.24 (0.056) \\ \cline{2-4} 
                      & 5       & 4.23 (0.41) & 0.08 (0.022) \\ \hline
\end{tabular}}
\caption{Statistics of binding affinities and novelty of LLM-generated molecules using
{\tt GPT-4o}. 
The entries represent the mean values, with standard deviations shown in parentheses.
$C_0$ denotes a set of known inhibitor of the target protein, provided as an early context 
to the LLM during prompting (see later).
Predicted affinity refers to the predicted binding affinity from {\tt GNINA} software.}
\label{appfig:context}
\end{table}

 Figures \ref{fig:param_s}--\ref{fig:param_temp} shows the effect of
 varying (in turn): sample-size in \GM;  and LLM ``temperature''. 
\GM with GPT-4o shows statistically significant
difference in predicted binding affinities, with
two different sample sizes ($s=10, 20$) for the
JAK2 and DBH proteins. Although having a high sample size
($s=20$) may result in better and more consistent
binding affinities, this was not observed for DRD2 protein.  
Figure~\ref{fig:param_temp} shows the effect of sampling
temperature on predicted binding affinities for molecules
generated by \GM using GPT-4o and Claude 3.5 Sonnet. 
For GPT-4o, temperature changes produce statistically significant
differences in mean affinity scores, whereas 
Molecule-GPT2 exhibits minimal variation across temperatures and statistically 
significant difference in number of molecules generated by \GM and
the mean affinity. We performed a similar investigation using
Claude 3.5 Sonnet and observe that it
remains largely stable across temperatures, with no
statistically significant variation in mean affinity.

    \begin{figure}[!htb]
        \centering
        \includegraphics[width=\textwidth]{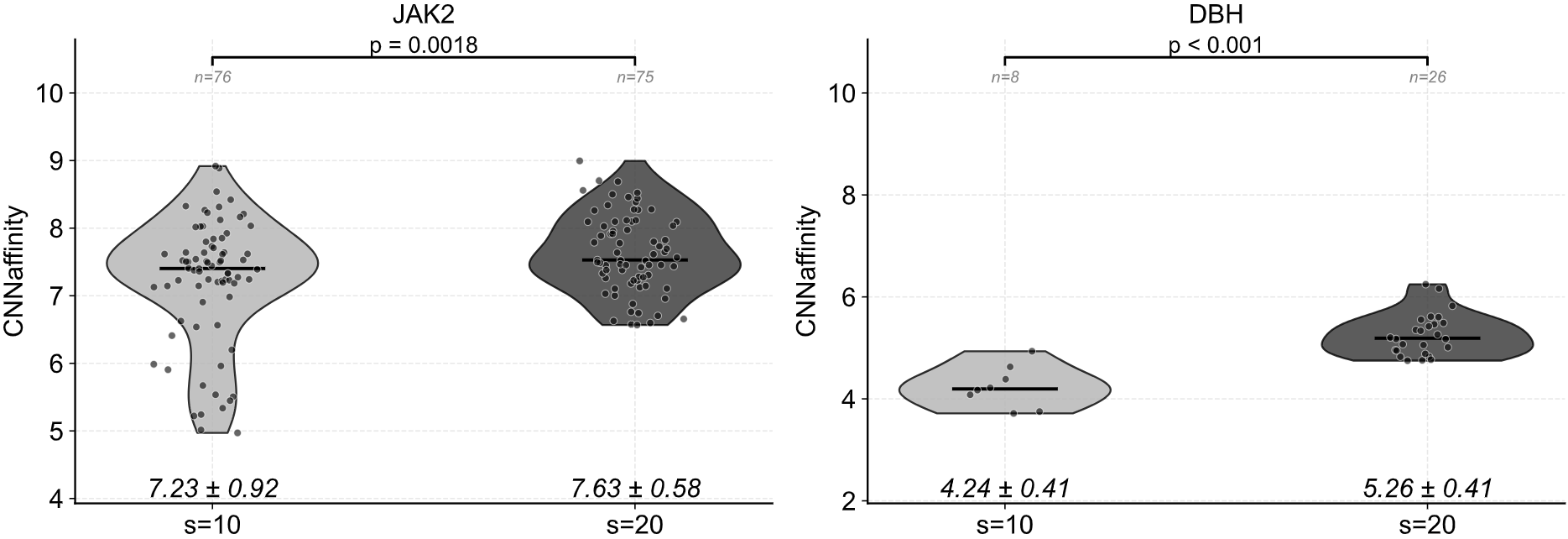}
        \caption{Effect of the parameter $s$ on \GM's performance in
        generating molecules for JAK2 and DBH proteins.
        The LLM used here is GPT-4o.}
        \label{fig:param_s}
    \end{figure}
    
    \begin{figure}[!htb]
        \centering
        \includegraphics[width=\textwidth]{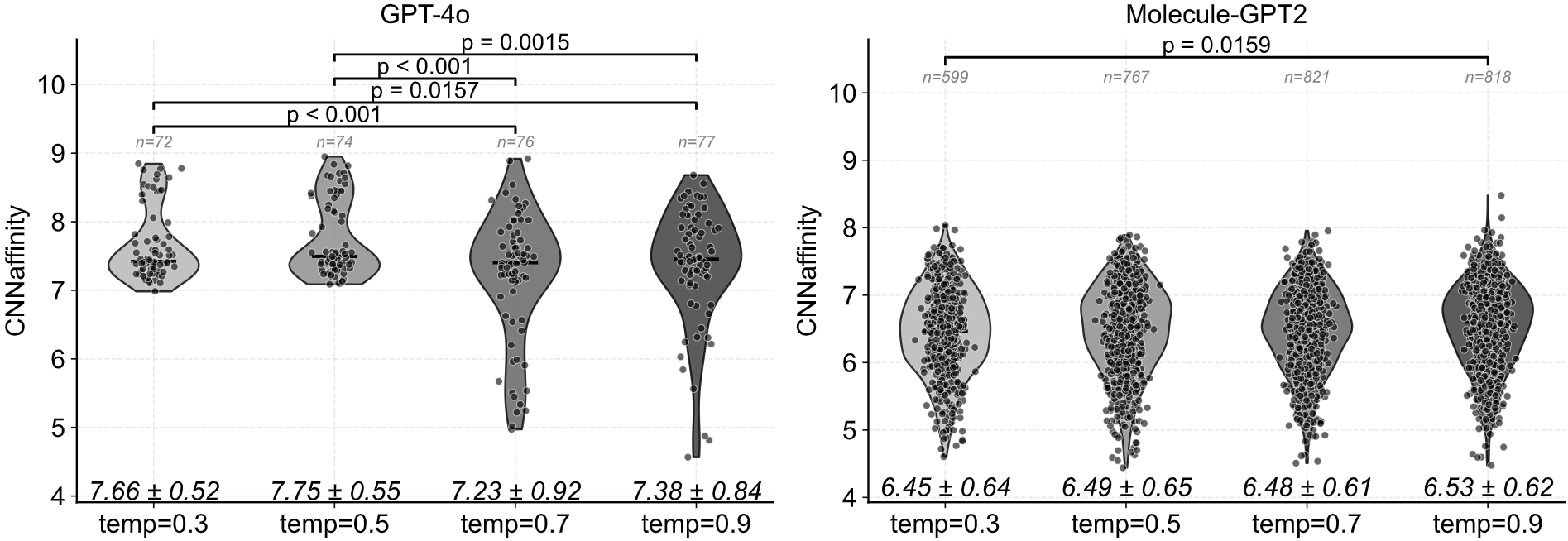}
        \caption{Effect of the temperature parameter while generating
        molecules using \GM with a general-purpose LLM,  GPT-4o (left) and a domain-specific LLM, Molecule-GPT2 (right). 
        $p$-value is computed
        using Welch's t-test comparing the mean CNNaffinity between two temperature settings.}
        \label{fig:param_temp}
    \end{figure}

\subsection{Prompts}
\label{app:prompts}

We distinguish between 2 types of prompts for the API calls to the LLM (in this paper, \texttt{GPT}): 
\begin{itemize}
\item System prompt: We use this prompt to guide the model's behaviour and
responses. It sets the overall instructions for the model, such as defining its role and
the syntactic format in which it should respond. In this work, we use this as:
``You are a scientist specialising in chemistry and drug design.
Your task is to generate valid SMILES strings as a comma-separated list inside square brackets. 
Return the response as plain text without any formatting, backticks, or explanations. 
The response must be formatted exactly as follows: [SMILES1, SMILES2, \dots]. 
Avoid any extra text or explanations.''
    
\item User prompt: This is the input provided by the user, containing the actual query
constructed in manner described below. The LLM generates responses based on this input 
while considering the instructions set by the system prompt. There are two kinds
of user prompts based on whether a set of inhibitors are shown to the LLM during search and generation or not. 
For revealing known inhibitors, we use: ``Generate up to $s$ novel valid molecules similar to the following positive molecules: [\dots]''. 
Otherwise, the prompt is simply ``Generate up to $s$ novel valid molecules''.
We also allow feasible molecules generated in \GM to be used as ``context''.
In this case, we use it as a part of the user prompt as: ``Additionally, consider these previously generated feasible molecules: [\dots].'' 
\end{itemize}

\subsection{Domain-specific LLM: Molecule-GPT2}

We implemented a variant of \GM using the publicly available 
GPT2 model hosted on Hugging Face \citep{entropy_gpt2_zinc_87m}. 
This model, which we call Molecule-GPT2, is based on the GPT2 architecture and has approximately 87 million parameters and trained with
about 480 million molecules (in SMILES representation)
from the ZINC database \citep{irwin2020zinc20}.
Unlike general-purpose LLMs such as GPT-4o and Claude 3.5 Sonnet, 
Molecule-GPT2 is not instruction-tuned and does not natively
support natural language prompting for molecule generation.
Therefore, we employed a scaffold-based prefixing strategy 
to condition the model's outputs on specific protein targets. 
For each target protein, we first identified 
a small set of representative molecular scaffolds by 
extracting common substructures from known inhibitors. 
In all our experiments,
we restrict to top 5 scaffolds, with each scaffold of maximum length 8. 
These scaffolds, encoded in SMILES format, were then used as 
fixed prefixes during autoregressive generation with the  GPT2 model. 
We performed molecule generation using HuggingFace's
$\mathtt{model.generate()}$ API with a maximum output 
length of 120 tokens, nucleus sampling with $p=0.9$,
temperature $=0.8$, and stochastic sampling enabled.
For each scaffold, we generated multiple candidate molecules
for downstream binding affinity evaluation for \GM.
Although this approach may not be the optimal one,
it does ensure that the generated molecules were chemically valid 
and retained structural motifs relevant to the target protein.

\subsection{Reproducibility and API cost}
A general concern with LLM-based systems is the reproducibility of results
obtained via closed, API-only commercial models. We therefore note the
approximate API costs incurred in this work. The entire study,
conducted over several months, incurred modest API costs: approximately
USD~10 each on OpenAI \texttt{GPT-4o} and Anthropic \texttt{Claude~3.5
Sonnet}, across roughly 400 requests and $\sim$100{,}000 generated tokens
in total, for each model. 
The domain-specific \texttt{Molecule-GPT2} model
\citep{entropy_gpt2_zinc_87m} was run locally and incurred no API cost.

We note that nothing in the SNG framework is specific to any particular
LLM: the \Gen procedure (Procedure~\ref{alg:nstar}) is agnostic to the
choice of generator, requiring only that it can be conditioned on a
hypothesis description. To strengthen reproducibility, we have included
results with a free, domain-specific \texttt{Molecule-GPT2} model
(Appendix~\ref{app:results}); a fully open general-purpose
instruction-tuned LLM (e.g.\ a \texttt{Llama-3} or \texttt{Mistral}
variant) could be dropped into the same interface. We identify
this as a concrete extension. All code and data are
available at the repository referenced in Sec.~\ref{sec:concl}.
\end{appendix}

\bibliography{refs}

\end{document}